\documentclass{article}

\usepackage{times}
\usepackage{epsfig}
\usepackage{graphicx}
\usepackage{amsmath}
\usepackage{amssymb}
\usepackage{paralist}
\usepackage{algorithm}
\usepackage{algorithmic}
\usepackage{booktabs}
\usepackage{rotating}

\usepackage{array}
\usepackage{color}
\usepackage{longtable}
\usepackage{booktabs}
\usepackage{multirow}

\usepackage[pagebackref=true,breaklinks=true,bookmarks=false]{hyperref}
\usepackage[letterpaper]{geometry}

\renewcommand{\topfraction}{.99}
\renewcommand{\floatpagefraction}{.99}%

\usepackage{amsthm}

\newtheorem{thm}{Theorem}
\newtheorem{lem}{Lemma}

\newtheorem{cor}{Corollary}
\newtheorem{definition}{Definition}
\providecommand{\tabularnewline}{\\}
\providecommand{\citep}{\cite}

\providecommand{\tabularnewline}{\\}

\title{Zen-NAS: A Zero-Shot NAS for High-Performance Image Recognition}

\author{Ming Lin \thanks{Accepted by ICCV 2021. Author home page \url{https://minglin-home.github.io}}\\
Alibaba Group\\
Bellevue, Washington, USA\\
{\tt\small ming.l@alibaba-inc.com}
\and
Pichao Wang\\
Alibaba Group\\
Bellevue, Washington, USA\\
{\tt\small pichao.wang@alibaba-inc.com}
\and
Zhenhong Sun\\
Alibaba Group\\
Hangzhou, Zhejiang, China\\
{\tt\small zhenhong.szh@alibaba-inc.com}
\and
Hesen Chen\\
Alibaba Group\\
Hangzhou, Zhejiang, China\\
{\tt\small hesen.chs@alibaba-inc.com}
\and
Xiuyu Sun\\
Alibaba Group\\
Hangzhou, Zhejiang, China\\
{\tt\small xiuyu.sxy@alibaba-inc.com}
\and
Qi Qian\\
Alibaba Group\\
Bellevue, Washington, USA\\
{\tt\small qi.qian@alibaba-inc.com}
\and
Hao Li\\
Alibaba Group\\
Hangzhou, Zhejiang, China\\
{\tt\small lihao.lh@alibaba-inc.com}
\and
Rong Jin\\
Alibaba Group\\
Hangzhou, Zhejiang, China\\
{\tt\small jinrong.jr@alibaba-inc.com}
}

\begin{document}

\maketitle

\begin{abstract}
  Accuracy predictor is a key component in Neural Architecture Search (NAS) for ranking architectures.  Building a high-quality accuracy predictor usually costs enormous computation. To address this issue, instead of using an  accuracy predictor, we propose a novel zero-shot index dubbed Zen-Score to rank the architectures. The Zen-Score represents the network expressivity and positively correlates with the model accuracy. The calculation of Zen-Score only takes a few forward inferences through a randomly initialized network, without training network parameters.   Built upon the Zen-Score, we further propose a new NAS algorithm, termed as Zen-NAS, by maximizing the Zen-Score of the target network under given inference budgets. Within less than half GPU day, Zen-NAS is able to directly search high performance architectures in a data-free style. Comparing with previous NAS methods, the proposed Zen-NAS is magnitude times faster on multiple server-side and mobile-side GPU platforms with state-of-the-art accuracy on ImageNet. Searching and training code as well as pre-trained models are available from \url{https://github.com/idstcv/ZenNAS}.
\end{abstract}


\section{Introduction}

\begin{figure}[!]
  \centering
  \includegraphics[width=\linewidth]{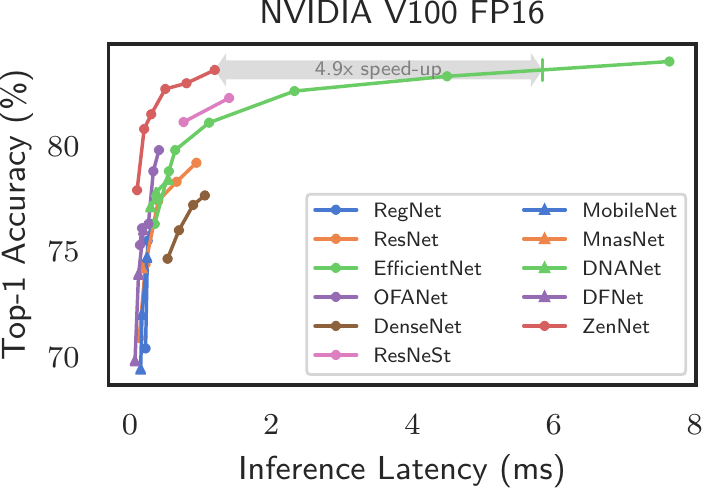}
  \caption{ZenNets top-1 accuracy v.s. inference latency (milliseconds per image) on ImageNet. Benchmarked on NVIDIA V100 GPU, half precision (FP16), batch size 64, searching cost 0.5 GPU day.} \label{fig:fig_latency_on_V100_pytorch_FP16}
\end{figure}

The design of high-performance deep neural networks is a challenging task. Neural Architecture Search (NAS) methods facilitate this progress. There are mainly two key components, architecture generator and accuracy predictor, in existing NAS algorithms. The generator proposes potential high-performance networks and the predictor predicts their accuracies. Popular generators include uniform sampling \citep{guoSinglePathOneShot2020}, evolutionary algorithm \citep{realRegularizedEvolutionImage2019} and reinforcement learning \citep{luoNeuralArchitectureOptimization2018}. The accuracy predictors include brute-force methods~\citep{realLargeScaleEvolutionImage2017,xieGeneticCNN2017,bakerDesigningNeuralNetwork2017,realRegularizedEvolutionImage2019}, predictor-based methods \cite{luoNeuralArchitectureOptimization2018,wenNeuralPredictorNeural2020,luoSemiSupervisedNeuralArchitecture2020} and one-shot methods~\citep{liuDARTSDifferentiableArchitecture2019,xuPCDARTSPartialChannel2019,zhouEcoNASFindingProxies2020, yangCARSContinuousEvolution2020,xieGeneticCNN2017,xieSNASStochasticNeural2018,cai_proxylessnas:_2019,zhangOvercomingMultiModelForgetting2020,wanFBNetV2DifferentiableNeural2020,caiOnceforAllTrainOne2020}.

A major challenge of building a high-quality accuracy predictor is the enormous computational cost. Both brute-forced methods and predictor-based methods require to train considerable number of networks. The one-shot methods reduce the training cost via parameter sharing. Albeit being more efficient than brute-forced methods, the one-shot methods still need to train a huge supernet which is still computationally expensive. Recent studies also find that nearly all supernet-based methods suffer from model interfering  \citep{caiOnceforAllTrainOne2020,yingNASBench101ReproducibleNeural2019} which degrades the quality of accuracy predictor \citep{sciutoEvaluatingSearchPhase2019}. In addition, since the supernet must be much larger than the target network, it is difficult to search large target networks under limited resources. These issues make the one-shot methods struggling in designing high-performance networks.

To solve these problems, instead of using an expensive accuracy predictor, we propose an almost zero-cost proxy, dubbed Zen-Score, for efficient NAS. The Zen-Score measures the expressivity \cite{pooleExponentialExpressivityDeep2016,maithraraghuExpressivePowerDeep2017} of a deep neural network and positively correlates with the model accuracy. The computation of Zen-Score only takes a few forward inferences on randomly initialized network using random Gaussian inputs, making it extremely fast, lightweight and data-free. Moreover, Zen-Score deals with the scale-sensitive problem caused by Batch Normalization (BN)\cite{bartlettSpectrallynormalizedMarginBounds2017,neyshaburRoleOverparametrizationGeneralization2018}, making it widely applicable to real-world problems. 


Based on Zen-Score, we design a novel Zen-NAS algorithm. It maximizes the Zen-Score of the target network within inference budgets. Zen-NAS is a \textbf{Zero-Shot} method since it does not optimize network parameters during search \footnote{Obviously, the final searched architecture must be trained on the target dataset before deployment.}. We apply Zen-NAS to search optimal networks under various inference budgets, including inference latency, FLOPs (Floating Point Operations) and model size, and achieve the state-of-the-art (SOTA) performance on CIFAR-10/CIFAR-100/ImageNet, outperforming previous human-designed and NAS-designed models by a large margin. Zen-NAS is the first zero-shot method that achieves SOTA results on large-scale full-resolution ImageNet-1k dataset \cite{dengImageNetLargescaleHierarchical2009} by the time of writing this work \cite{mellorNeuralArchitectureSearch2021,abdelfattahZeroCostProxiesLightweight2021,chen2021neural}.

Our approach is inspired by recent advances in deep learning studies \citep{montufarNumberLinearRegions2014,danielyDeeperUnderstandingNeural2016,liangWhyDeepNeural2016a,pooleExponentialExpressivityDeep2016,cohenInductiveBiasDeep2017,luExpressivePowerNeural2017,maithraraghuExpressivePowerDeep2017,rolnickPowerDeeperNetworks2018,serraBoundingCountingLinear2018,haninComplexityLinearRegions2019,xiongNumberLinearRegions2020} which show that deep models are superior than shallow ones since deep models are more expressive under the same number of neurons. According to the bias-variance trade-off in statistical learning theory \cite{koltchinskii2011oracle}, increasing the expressivity of a deep network implies smaller bias error. When the size  $n$ of training dataset is large enough, the variance error will diminish as $\mathcal{O}(1/\sqrt{n}) \rightarrow 0$. This means that the generalization error is dominated by the bias error which could be reduced by more expressive networks. These theoretical results are well-aligned with large-scale deep learning practices\citep{nguyenWideDeepNetworks2021,touvronTrainingDataefficientImage2021,phamMetaPseudoLabels2021}. 

We summarize our main contributions as follows:
\begin{compactitem}
 \item We propose a novel zero-shot proxy Zen-Score for NAS. The proposed Zen-Score is computationally efficient and is proved to be scale-insensitive in the present of BN. A novel NAS algorithm termed Zen-NAS is proposed to search for networks with maximal Zen-Score  in the design space.
 \item Within half GPU day, the ZenNets designed by Zen-NAS achieve up to $83.6\%$ top-1 accuracy on ImageNet that is as accurate as EfficientNet-B5 with inference speed magnitude times faster on multiple hardware platforms. To our best knowledge, \textbf{Zen-NAS is the first zero-shot method that outperforms training-based methods on ImageNet.}
\end{compactitem}

\section{Related Work}
\label{sec:related-work}

We briefly review the related works. For comprehensive review of NAS, the monograph \citep{ren_comprehensive_2020} is referred to. 

In the early days of NAS, brute-force methods are adopted to search architectures by directly training a network to obtain its accuracy. For example, the AmoebaNet \citep{realRegularizedEvolutionImage2019} conducts structure search on CIFAR-10 using Evolutionary Algorithm (EA) \citep{krizhevskyLearningMultipleLayers2009} and then transfers the structure to ImageNet. It takes about 3150 GPU days of searching and achieves $74.5\%$ top-1 accuracy on ImageNet. Inspired by the success of AmoebaNet, many EA-based NAS algorithms are proposed to improve the searching efficiency, such as EcoNAS \citep{zhouEcoNASFindingProxies2020}, CARS \citep{yangCARSContinuousEvolution2020}, GeNet \citep{xieGeneticCNN2017} and  PNAS \citep{liuProgressiveNeuralArchitecture2018}. These methods search on down-sampled images or reduce the number of queries. Reinforced Learning is another popular generator (sampler) in NAS, including NASNet \citep{zophLearningTransferableArchitectures2018}, Mnasnet \citep{tan_mnasnet:_2019} and MetaQNN \citep{bakerDesigningNeuralNetwork2017}. 


Both EA and RL based methods require lots of network training. To address this problem, the predictor-based methods encode architectures into high dimensional vectors. A number of architectures are trained to obtain their accuracies \citep{luoNeuralArchitectureOptimization2018,luoSemiSupervisedNeuralArchitecture2020} and then are used as training data for learning accuracy predictor. The one-shot methods further reduce the training cost by training a big supernet. This framework is widely applied in many efficient NAS methods, including DARTS \citep{liuDARTSDifferentiableArchitecture2019}, SNAS \citep{xieSNASStochasticNeural2018}, PC-DARTS \citep{xuPCDARTSPartialChannel2019}, ProxylessNAS \citep{cai_proxylessnas:_2019}, GDAS \citep{zhangOvercomingMultiModelForgetting2020}, FBNetV2 \citep{wanFBNetV2DifferentiableNeural2020},  DNANet \citep{liBlockwiselySupervisedNeural2020}, Single-Path One-Shot NAS \citep{guoSinglePathOneShot2020}.

Although the above efforts have greatly reduced the searching cost, their top-1 accuracies on ImageNet are below $80.0\%$. The authors of OFANet \citep{caiOnceforAllTrainOne2020} noted that weight-sharing suffers from model interfering. They propose a progressive-shrinking strategy to address this issue. The resultant OFANet achieves $80.1\%$ accuracy after searching for 51.6 GPU days. EfficientNet \citep{tanEfficientNetRethinkingModel2019} is another high precision network designed by NAS. It takes about 3800 GPU days to search EfficientNet-B7 whose accuracy is $84.4\%$. In comparison, Zen-NAS achieves $83.6\%$ accuracy while using magnitude times fewer resources.



A few on-going works are actively exploring zero-shot proxies for efficient NAS. However, these efforts have not delivered the SOTA results. In a recent empirical study, \cite{abdelfattahZeroCostProxiesLightweight2021} evaluates the performance of six zero-shot pruning proxies on NAS benchmark datasets. The synflow \cite{tanakaPruningNeuralNetworks2020} achieves best results in their experiments. We compare synflow with Zen-Score under fair settings and show that Zen-Score achieves $+1.1\%$ better accuracy on CIFAR-10 and $+8.2\%$ better accuracy on CIFAR-100. The concurrent work TE-NAS~\cite{chen2021neural} uses  a combination of NTK-score and network expressivity as NAS proxy. Specifically, the TE-NAS estimates the expressivity by directly counting the number of active regions $R_N$ on randomly sampled images. In comparison, Zen-Score not only considers the distribution of linear regions but also considers the Gaussian complexity of linear classifier in each linear region, giving a more accurate estimation of network expressivity. 
The computation of Zen-Score is 20 to 28 times faster than TE-NAS score. In terms of performance, TE-NAS achieves $74.1\%$ top-1 accuracy on ImageNet, lagging behind SOTA baselines. Zen-NAS achieves $+9.5\%$ better accuracy within similar searching cost. Another concurrent work NASWOT \cite{mellorNeuralArchitectureSearch2021a} computes the architecture score according to the kernel matrix of binary activation patterns between mini-batch samples. It achieves similar top-1 accuracies on CIFAR-10/CIFAR-100 as TE-NAS.

It is important to distinguish Zen-NAS from unsupervised NAS (UnNAS) \cite{liuAreLabelsNecessary2020a}. In UnNAS, the network is trained to predict the pre-text tasks therefore it still requires parameter training. In Zen-NAS, no parameter training is required during the search.

In this work, we mostly focus on the vanilla network space described in the next section. Several previous works design networks in a more general irregular design space, such as DARTS \cite{liuDARTSDifferentiableArchitecture2019} and RandWire \cite{xieExploringRandomlyWired2019}. Zen-NAS cannot be applied to these irregular design spaces since Zen-Score is not mathematically well-defined in irregular design spaces. In practice, the vanilla network space is a large enough space which covers most SOTA networks, including but not limited to ResNet, MobileNet and EfficientNet. Particularly, Zen-NAS outperforms DARTS-based methods by a significant margin on ImageNet.

\section{Expressivity of Vanilla Network}
\label{sec:Expressivity-of-Vanilla-Network}

In this section, we discuss how to measure the expressivity of \emph{vanilla convolutional neural network}  (VCNN) family, an ideal prototype for theoretical studies. We show that the expressivity of a network can be efficiently measured by its expected Gaussian complexity, or $\Phi$-score for short. In the next section, we further show that for very deep networks, directly computing $\Phi$-score incurs numerical overflow. This overflow can be addressed by adding BN layers and then re-scaling the $\Phi$-score by a constant. This new score is named as Zen-Score in Section \ref{sec:zen-score-and-zen-nas}.

\subsection{Notations}

An $L$-layer neural network is formulated as a function $f: \mathbb{R}^{m_0} \rightarrow \mathbb{R}^{m_{L}}$ where $m_0$ is the input dimension and $m_L$ is the output dimension. $\boldsymbol{x}_0 \in \mathbb{R}^{m_0}$ denotes the input image. Correspondingly, the output feature map of the $t$-th layer is denoted by $\boldsymbol{x}_t$.  The $t$-th layer has $m_{t-1}$ input
channels and $m_{t}$ output channels. The convolutional kernel is $\boldsymbol{\theta}_{t}\in\mathbb{R}^{m_{t}\times m_{t-1}\times k\times k}$. The image resolution is $H\times W$. The mini-batch size is $B$. The Gaussian distribution of mean $\mu$ and variance $\sigma^2$ is denoted by $\mathcal{N}({\mu}, {\sigma})$. 

\subsection{Vanilla Convolutional Neural Network}

The \textbf{vanilla convolutional neural network} (VCNN) is a widely used prototype in theoretical  studies \cite{pooleExponentialExpressivityDeep2016,serraBoundingCountingLinear2018,haninComplexityLinearRegions2019}. The main body of a vanilla network is stacked by multiple convolutional layers. Each layer consists of one convolutional operator followed by RELU activation. All other components are removed from the backbone, including residual link and Batch Normalization. After the main body, global average pool layer (GAP) reduces the feature map resolution to 1x1, followed by a fully-connected layer. At the end a soft-max operation converts the network output to label prediction. Given the input $\boldsymbol{x}$ and network parameters $\boldsymbol{\theta}$,  $f(\boldsymbol{x} | \boldsymbol{\theta})$ refers to the output of the main body of the network, that is the feature map before the GAP layer (pre-GAP layer). We measure the network expressivity with respect to pre-GAP because it contains most of the information we need.

Modern networks use auxiliary structures such as residual link , Batch Normalization and self-attention block \cite{huSqueezeandExcitationNetworks2018}. These structures will not significantly affect the representation power of networks. For example, BN layer can be merged into convolutional kernel via kernel fusion. Self-attention linearly combines existing feature maps hence spans the same subspace. Therefore, these structures are temporarily removed  when measuring network expressivity and then added back in training and testing stages. For non-RELU activation functions, they are replaced by RELU in a similar way. These simple modifications make our method applicable to a majority of non-VCNN models widely used in practice. In fact, nearly all single-branch feed-forward networks can be converted to vanilla network by the aforementioned modifications.

\subsection{$\Phi$-Score as Proxy of Expressivity}

Given a VCNN $f(\boldsymbol{x} | \boldsymbol{\theta})$,  we propose a novel numerical index $\Phi$-score as a proxy of its expressivity.  The definition of $\Phi$-score is inspired by recent theoretical studies on deep network expressivity \cite{serraBoundingCountingLinear2018,xiongNumberLinearRegions2020}. A key observation in these studies is that a vanilla network can be decomposed into piece-wise linear functions conditioned on  activation patterns \cite{montufarNumberLinearRegions2014}:

\begin{lem}[\cite{montufarNumberLinearRegions2014,maithraraghuExpressivePowerDeep2017}]
  \label{prop:vcnn-is-linear}
  Denote the activation pattern of the $t$-th layer as $\mathcal{A}_t(\boldsymbol{x})$. Then for any vanilla network $f(\cdot)$,
  \begin{align}
     \label{eq:all-vcnn-are-linear-functions}
     f(\boldsymbol{x}|\boldsymbol{\theta}) = \sum_{\mathcal{S}_i \in \mathcal{S}} \mathcal{I}_{\boldsymbol{x}}(\mathcal{S}_i) \boldsymbol{W}_{\mathcal{S}_i} \boldsymbol{x}
  \end{align}
  where $\mathcal{S}_i$ is a convex polytope depending on $\{\mathcal{A}_1(\boldsymbol{x}), \mathcal{A}_2(\boldsymbol{x}), \cdots, \mathcal{A}_L(\boldsymbol{x})\}$; $\mathcal{S}$ is a finite set of convex polytopes in $\mathbb{R}^{m_0}$;  $\mathcal{I}_{\boldsymbol{x}}(\mathcal{S}_i) = 1$ if $\boldsymbol{x} \in \mathcal{S}_i$ otherwise zero; $\boldsymbol{W}_{\mathcal{S}_i}$ is a coefficient matrix of size $\mathbb{R}^{m_L \times m_0}$.
\end{lem}



According to Lemma \ref{prop:vcnn-is-linear}, any vanilla network is an ensemble of piece-wise linear functions segmented by convex polytopes $\mathcal{S}=\{\mathcal{S}_1, \mathcal{S}_2, \cdots, \mathcal{S}_{|\mathcal{S}|}\}$ where $|\mathcal{S}|$ is the number of linear-regions (see Figure 2 in \cite{haninComplexityLinearRegions2019}. The number of linear regions $|\mathcal{S}|$ has been used as expressivity proxy in several theoretical studies \cite{montufarNumberLinearRegions2014,serraBoundingCountingLinear2018,haninComplexityLinearRegions2019,zhangEmpiricalStudiesProperties2019,xiongNumberLinearRegions2020}. However, directly using $|\mathcal{S}|$ incurs two limitations: a) Counting $|\mathcal{S}|$ for large network is computationally infeasible; b) The representation power of each $W_{\mathcal{S}_i}$ is not considered in the proxy. The first limitation is due to fact that the  number of linear regions grow exponentially for large networks \cite{montufarNumberLinearRegions2014,xiongNumberLinearRegions2020}. To understand the second limitation, we recall the Gaussian complexity \cite{kakadeComplexityLinearPrediction2008} of linear classifiers:

\begin{lem} [\cite{kakadeComplexityLinearPrediction2008}]
  \label{prop:gaussian-complexity-of-linear-classifier}
  For linear function class $\{f: f(X) = W X \ \mathrm{s.t.} \ \|W\|_F \leq G \} $, its Gaussian complexity is upper bounded by $\mathcal{O}(G)$.
\end{lem}

In other words, Lemma \ref{prop:gaussian-complexity-of-linear-classifier} says that the expressivity of linear function class measured by Gaussian complexity is controlled by the Frobenius norm of its  parameter matrix $W$. Inspired by Lemma \ref{prop:vcnn-is-linear} and Lemma \ref{prop:gaussian-complexity-of-linear-classifier}, we define the following index for measuring  network expressivity :

\begin{definition}[$\Phi$-score for VCNN]
  \label{def:Expected-Gaussian-Complexity-for-VCNN}
  The expected Gaussian complexity for a vanilla network $f(\cdot)$ is defined by
  \begin{align}
     \Phi(f) =& \log \mathbb{E}_{\boldsymbol{x},\boldsymbol{\theta}} \left \{ 
        \sum_{\mathcal{S}_i \in \mathcal{S}} \mathcal{I}_{\boldsymbol{x}}(\mathcal{S}_i) \| \boldsymbol{W}_{\mathcal{S}_i} \|_F \right \} \\
        =& \log \mathbb{E}_{\boldsymbol{x},\boldsymbol{\theta}} \| \nabla_{\boldsymbol{x}} f(\boldsymbol{x} | \boldsymbol{\theta}) \|_F.
  \end{align}
\end{definition}

In Definition \ref{def:Expected-Gaussian-Complexity-for-VCNN}, we measure the network expressivity by its expected Gaussian complexity, or $\Phi$-score for short. Since any VCNN is ensemble of linear functions, it is nature to measure its expressivity by averaging the Gaussian complexity of linear function in each linear region. To this end, we randomly sample $\boldsymbol{x}$ and $\boldsymbol{\theta}$ from some prior distributions and then average $\|W_{\mathcal{S}_i}\|_F$. This is equivalent to compute the expected gradient norm of $f$ with respect to input $\boldsymbol{x}$. In our implementation, $\boldsymbol{x}$ and $\boldsymbol{\theta}$ are sampled from standard Gaussian distribution which works well in practice. It is important to note that in $\Phi$-score, only the gradient of $\boldsymbol{x}$ rather than $\boldsymbol{\theta}$ is involved. This is different to zero-cost proxies in \cite{abdelfattahZeroCostProxiesLightweight2021} which compute gradient of $\boldsymbol{\theta}$ in their formulations. These proxies  measure the trainability  \cite{wangPickingWinningTickets2019,chen2021neural} instead of the expressivity of networks.

\section{Zen-Score and Zen-NAS}
\label{sec:zen-score-and-zen-nas}

\begin{figure*}[!]
  \begin{minipage}{0.32\linewidth}
    \begin{center}
      \includegraphics[width=\linewidth]{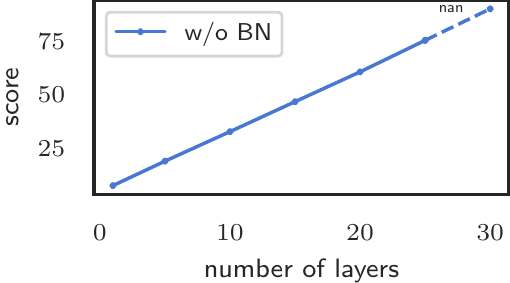}\\
      (a) $\Phi$-score of $\mathrm{P}_{\mathrm{w/oBN}}$ networks
    \end{center}
  \end{minipage}
  \hfil
  \begin{minipage}{0.32\linewidth}
    \centering
    \begin{center}
      \includegraphics[width=\linewidth]{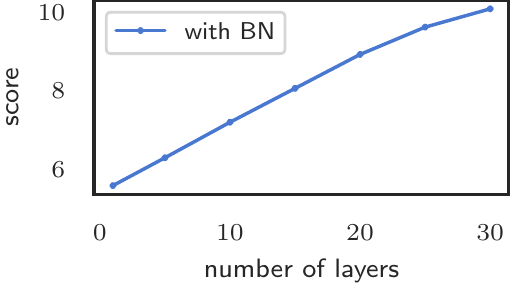}\\
      (b) $\Phi$-score of $\mathrm{P}_{\mathrm{BN}}$ networks
    \end{center}    
  \end{minipage}
  \hfil
  \begin{minipage}{0.32\linewidth}
    \centering
    \begin{center}
      \includegraphics[width=\linewidth]{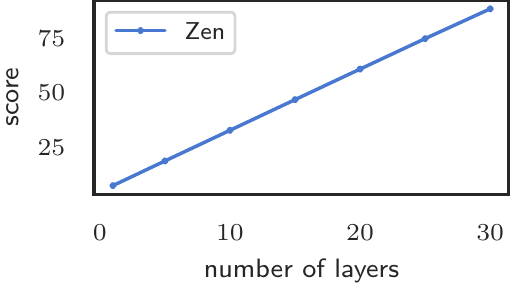}\\
      (c) Zen-Score of $\mathrm{P}_{\mathrm{BN}}$ networks
    \end{center}    
  \end{minipage} \\
  \begin{minipage}{0.32\linewidth}    
    \begin{center}
      \includegraphics[width=\linewidth]{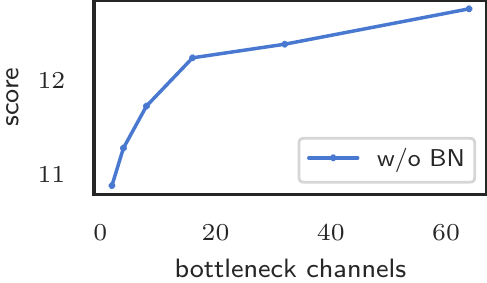}\\
      (d) $\Phi$-score of $\mathrm{Q}_{\mathrm{w/oBN}}$ networks
    \end{center}    
  \end{minipage}
  \hfil 
  \begin{minipage}{0.32\linewidth}
    \centering
    \begin{center}
      \includegraphics[width=\linewidth]{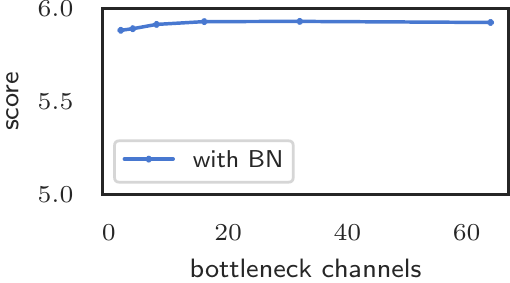}\\
      (e) $\Phi$-score of $\mathrm{Q}_{\mathrm{BN}}$ networks
    \end{center}    
  \end{minipage}
  \hfil 
  \begin{minipage}{0.32\linewidth}
    \centering
    \begin{center}
      \includegraphics[width=\linewidth]{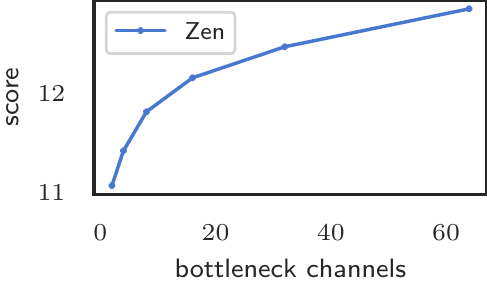}\\
      (f) Zen-Score of $\mathrm{Q}_{\mathrm{BN}}$ networks
    \end{center}    
  \end{minipage}      
  \caption{$\Phi$-scores and Zen-Scores of networks, with different depths and bottleneck channels.}
  \label{fig:delta-zen-score-vs-channel-layer}
\end{figure*}

In this section, we show that directly computing $\Phi$-score for very deep networks incurs numerical overflow due to the gradient explosion without BN layers. The gradient explosion could be resolved by adding BN layers back but the $\Phi$-score will be adaptively re-scaled, making it difficult to compare $\Phi$-score between different networks. The same phenomenon has been known as `scale-sensitive' problem in deep learning complexity analysis \cite{bartlettSpectrallynormalizedMarginBounds2017,neyshaburRoleOverparametrizationGeneralization2018}. To address this open question, we propose to re-scale the $\Phi$-score one more time by the product of BN layers' variance statistics. This new score is denoted as Zen-Score in order to distinguish from the original $\Phi$-score. The Zen-Score is proven to be scale-insensitive. Finally, we present Zen-NAS algorithm built on Zen-Score and demonstrate its effectiveness in the next section.

\subsection{Overflow and BN-rescaling}

When computing $\Phi$-score for very deep vanilla networks, numerical overflow incurs almost surely. This is because  BN layers are removed from the network and the magnitude of network output grows exponentially along depth. To see this, we construct a set of vanilla networks $P_\mathrm{w/oBN}$ without BN layers. All networks have the same widths but different depths. Figure~\ref{fig:delta-zen-score-vs-channel-layer}(a) plots the $\Phi$-scores for $P_\mathrm{w/oBN}$. After 30 layers, $\Phi$-score overflows. To address the overflow, we add BN layers back and compute the $\Phi$-scores in  Figure~\ref{fig:delta-zen-score-vs-channel-layer}(b). This time the overflow  dismisses but the $\Phi$-scores are scaled-down by a large factor. This phenomenon is termed as BN-rescaling.

To demonstrate that BN-rescaling disturbs architecture ranking, we construct another two set of networks, $Q_\mathrm{w/oBN}$ and $Q_\mathrm{BN}$, with and without BN respectively. All networks have two layers and have the same number of input and final output channels. The number of bottleneck channels, that is the width of the hidden layer, varies from 2 to 60. The corresponding $\Phi$-score curves are plotted in Figure~\ref{fig:delta-zen-score-vs-channel-layer}(d) and (e) respectively. When BN layer is presented, the $\Phi$-score becomes nearly constant for all networks. This will confuse the architecture generator and drive the search to a wrong direction.

\subsection{From $\Phi$-Score to Zen-Score}

\begin{algorithm}[!]
  \caption{Zen-Score}
  \label{alg:zen-score}
\begin{algorithmic}[1]

\REQUIRE Network $\mathcal{F}(\cdot)$ with pre-GAP feature map $f(\cdot)$; $\alpha=0.01$.

\ENSURE Zen-Score $\mathrm{Zen}(\mathcal{F})$.

\STATE Remove all residual links in $\mathcal{F}$.

\STATE Initialize all neurons in $\mathcal{F}$ by $\mathcal{N}(0,1)$.

\STATE Sample $\boldsymbol{x}, \boldsymbol{\epsilon} \sim \mathcal{N}(0,1)$.

\STATE Compute $\Delta \triangleq \mathbb{E}_{\boldsymbol{x}, \boldsymbol{\epsilon}} \| f(\boldsymbol{x}) - f(\boldsymbol{x} + \alpha \boldsymbol{\epsilon})\|_F $ \, .

\STATE For the $i$-th BN layer with $m$ output channels, compute $\bar{\sigma}_i=\sqrt{\sum_j \boldsymbol{\sigma}_{i,j}^2/m}$ where $\boldsymbol{\sigma}_{i,j}$ is the mini-batch standard deviation statistic of the $j$-th channel in BN.

\STATE $\mathrm{Zen}(F) \triangleq \log(\Delta) + \sum_i \log(\bar{\sigma}_i)$.

\end{algorithmic}
\end{algorithm}

\begin{figure*}[!]
  \begin{center}
    \includegraphics[width=\linewidth]{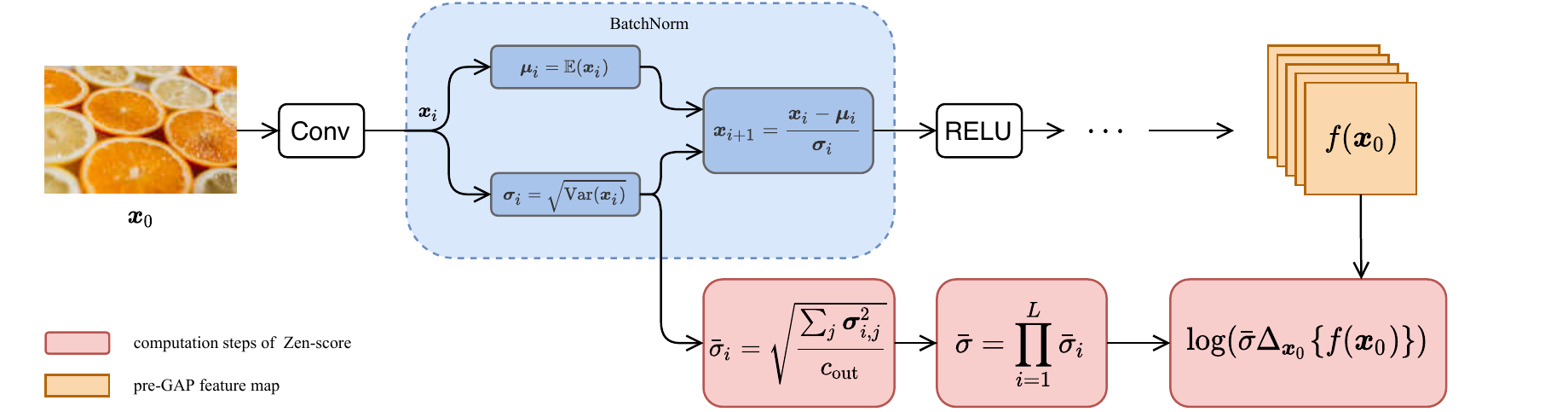}  
  \end{center}
\caption{Zen-Score computational graph. $\boldsymbol{x}_0$ is one mini-batch of input images. For each BN layer, we extract its mini-batch deviation parameter $\boldsymbol{\sigma}_i$. $\Delta_{\boldsymbol{x}_0}\{f(\boldsymbol{x}_0)\}$ is the differential of pre-GAP feature map $f(\boldsymbol{x}_0)$ with respect to $\boldsymbol{x}_0$.}
\label{fig:fig_compute_zen_score}
\end{figure*}

In the above subsection, we showed that BN layer is necessary to prevent numerical overflow in computing $\Phi$-score but comes with the side-effect of re-scaling. In this subsection, we design a new Zen-Score which is able to calibrate re-scaling when BN layer is present. The computation  of Zen-Score is described in Algorithm \ref{alg:zen-score}. Figure \ref{fig:fig_compute_zen_score} visualizes the computational graph of Algorithm \ref{alg:zen-score}.

In Algorithm~\ref{alg:zen-score}, all residual links are removed from the network as pre-processing. Then we randomly sample input vectors and perturb them with Gaussian noise. The perturbation of the pre-GAP feature map is denoted as $\Delta$ in Line 4. This step replaces the gradient of $\boldsymbol{x}$ with finite differential $\Delta$ to avoid backward-propagation. 
To get Zen-Score, the scaling factor $\bar{\sigma}_i^2$ is averaged from the variance of each channel in BN layer. Finally, the Zen-Score is computed by the log-sum of $\Delta$ and $\bar{\sigma}_i$. The following theorem guarantees that the Zen-Score of network with BN layers approximates the $\Phi$-score of the same network without BN layers. The proof is postponed to Supplementary~\ref{sec:Proof-of-Theorem-1}.
\begin{thm}
 \label{thm:zen-score-equal-phi-score}
 Let $\bar{f}(\boldsymbol{x}_0)=\bar{\boldsymbol{x}}_{L}$ be an $L$-layer vanilla network without BN layers. $f(\boldsymbol{x}_0)=\boldsymbol{x}_{L}$ is its sister network with BN layers. For some constants $0<\delta<1$, $K_0\leq \mathcal{O}[\sqrt{\log(1/\delta)}]$, when $BHW \geq \mathcal{O}[(L K_0)^2]$ is large enough, with probability at least $1-\delta$, we have
 \begin{align}
   \label{eq:zen-vs-phi-ratio-lower-upper-bound}
   (1-L\epsilon)^{2} \leq \frac{(\prod_{t=1}^{L}\bar{\sigma}_{t}^{2})\mathbb{E}_{\theta}\{\|\boldsymbol{x}_{L}\|^{2}\}}{\mathbb{E}_{\theta}\|\bar{\boldsymbol{x}}_{L}\|^{2}} \leq (1+L\epsilon)^2 
 \end{align}
where $\epsilon \triangleq \mathcal{O}(2K_{0}/\sqrt{BHW})$.
\end{thm}

Informally speaking, Theorem \ref{thm:zen-score-equal-phi-score} says that to compute $\|\bar{f}(\cdot)\|$, we only need to compute $\|f(\cdot)\|$ then re-scale with $\prod_{t=1}^{L}\bar{\sigma}_{t}$. The approximation error is bounded by $L\epsilon$. By taking gradient of $\boldsymbol{x}$ on both $\bar{f}(\cdot)$ and ${f}(\cdot)$, we obtain the desired relationship between Zen-Score and $\Phi$-score.

\subsection{Zen-NAS For Maximizing Expressivity}

\begin{algorithm}
 \caption{Zen-NAS}
 \label{alg:zen-nas}
\begin{algorithmic}[1]

\REQUIRE Search space $\mathcal{S}$, inference budget $B$, maximal depth $L$, total number of iterations $T$, evolutionary population size $N$, initial structure $F_0$.

\ENSURE NAS-designed ZenNet $F^*$.

\STATE Initialize population $\mathcal{P}=\{F_0\}$.

\FOR{$t=1,2,\cdots,T$}
\STATE Randomly select $F_t \in \mathcal{P}$.
\STATE Mutate $\hat{F}_t=\textrm{MUTATE}(F_t, \mathcal{S})$
\IF{$\hat{F}_t$ exceeds inference budget or has more than $L$ layers} 
 \STATE Do nothing.
\ELSE
 \STATE Get Zen-Score $z=\textrm{Zen}(\hat{F}_t)$.
 \STATE Append $\hat{F}_t$ to $\mathcal{P}$.
\ENDIF

\STATE Remove network of the smallest Zen-Score if the size of $\mathcal{P}$ exceeds $B$.
\ENDFOR

\STATE Return $F^*$, the network of the highest Zen-Score in $\mathcal{P}$.

\end{algorithmic}
\end{algorithm}

\begin{algorithm}
 \caption{MUTATE}
 \label{alg:mutate}
\begin{algorithmic}[1]

\REQUIRE Structure $F_t$, search space $\mathcal{S}$. 

\ENSURE Randomly mutated structure $\hat{F}_t$.
\STATE Uniformly select a block $h$ in $F_t$.
\STATE Uniformly alternate the block type, kernel size, width and depth of $h$ within some range.
\STATE Return the mutated structure $\hat{F}_t$.
\end{algorithmic}
\end{algorithm}

We design Zen-NAS algorithm to maximize the Zen-Score of the target network. The step-by-step description of Zen-NAS is given in Algorithm~\ref{alg:zen-nas}. The Zen-NAS uses Evolutionary Algorithm (EA) as architecture generator. It is possible to choose other generators such as Reinforced Learning or even greedy selection. The choice of EA is due to its simplicity.

In Algorithm~\ref{alg:zen-nas}, we randomly generate $N$ structures. At each iteration step $t$, we randomly select a structure in the population $\mathcal{P}$ and mutate it. The mutation algorithm is presented in Algorithm~\ref{alg:mutate}. The width and depth of the selected layer is mutated in a given range. We choose $[0.5, 2.0]$ as the mutation range in this work, that is, within half or double of the current value. The new structure $\hat{F}_t$ is appended to the population if its inference cost does not exceed the budget. The maximal depth of networks is controlled by $L$, which prevents the algorithm generate over-deep structures. Finally, we maintain the population size by removing  networks with the smallest Zen-Scores. After $T$ iterations, the network with the largest Zen-Score is returned as the output of Zen-NAS. We name the found architectures as ZenNets.

\section{Experiments}
\label{sec:experiment}

In this section, experiments on CIFAR-10/CIFAR-100 \citep{krizhevskyLearningMultipleLayers2009} and ImageNet-1k \citep{dengImageNetLargescaleHierarchical2009} are conducted to validate the superiority of Zen-NAS. We first compare Zen-Score to several zero-shot proxies on CIFAR-10 and CIFAR-100, using the same search space,  search policy and training settings. Then we compare Zen-NAS to the state-of-the-art methods on ImageNet. Zen-NAS on CIFAR-10/CIFAR-100 can be found in Supplementary \ref{sec:Zen-NAS-on-CIFAR}. Finally, we compare the searching cost of Zen-NAS with SOTA methods in subsection \ref{sec:Searching-Cost-of-Zen-NAS}. 

Due to space limitation, the inference speed on NVIDIA T4 and Google Pixel2 is reported in Supplementary \ref{sec:additional-figures}. The Zen-Scores of ResNets and accuracies under fair training settings are reported in Supplementary \ref{sec:zen-score-accuracy-of-resnets-under-fair-training}. We enclose one big performance table of networks on ImageNet in Supplementary \ref{sec:one-big-table}.

To align with previous works, we consider the following two search spaces:
\begin{compactitem}
 \item \textbf{Search Space I} Following \citep{heDeepResidualLearning2016,radosavovicDesigningNetworkDesign2020}, this search space consists of residual blocks and bottleneck blocks defined in ResNet.
 \item \textbf{Search Space II} Following \citep{sandlerMobileNetV2InvertedResiduals2018,phamEfficientNeuralArchitecture2018a}, this search space consists of MobileNet blocks. The depth-wise expansion ratio is searched in set $\{1,2,4,6\}$.
\end{compactitem}
Please see Supplementary \ref{sec:Datasets-and-Experiment-Settings} for datasets description and detail experiment settings. 


\begin{table}[!]
 \begin{center}
 \begin{tabular}{lcc}
     \toprule 
     proxy & CIFAR-10 & CIFAR-100\tabularnewline
     \midrule
     \midrule 
     Zen-Score & \textbf{96.2\%} & \textbf{80.1\%} \tabularnewline
     \midrule 
     FLOPs & 93.1\% & 64.7\%\tabularnewline
     \midrule 
     grad & 92.8\% & 65.4\%\tabularnewline
     \midrule 
     synflow & 95.1\% & 75.9\%\tabularnewline
     \midrule 
     TE-Score & 96.1\% & 77.2\%\tabularnewline
     \midrule
     NASWOT & 96.0\% & 77.5\%\tabularnewline
     \midrule 
     Random  & 93.5$\pm$0.7\% & 71.1$\pm$3.1\%\tabularnewline
     \bottomrule
     \end{tabular}
 \end{center}
 \caption{Top-1 accuracies on CIFAR-10/CIFAR-100 for five zero-shot proxies. Budget: model size $N\leq 1\,\mathrm{M}$. `Random': average accuracy $\pm$ std for random search.}
 \label{tab:top1-acc-for-five-proxy}
\end{table}

\begin{table}[!]
 \begin{center}
 \begin{tabular}{llccc}
 \toprule 
 proxy & model & N & time & speed-up\tabularnewline
 \midrule
 \midrule 
 TE-Score & ResNet-18 & 16 & 0.34 & 1/28x \tabularnewline
 \midrule 
  & ResNet-50 & 16 & 0.77 & 1/20x \tabularnewline
 \midrule 
 NASWOT$\dagger$ & ResNet-18 & 16 & 0.040 & 1/3.3x \tabularnewline
 \midrule 
  & ResNet-50 & 16 & 0.059 & 1/1.6x \tabularnewline
 \midrule 
 Zen-Score & ResNet-18 & 16 & 0.012 & 1.0 \tabularnewline
 \midrule 
  & ResNet-50 & 16 & 0.037 & 1.0 \tabularnewline
 \bottomrule
 \end{tabular}
 \end{center}
 \caption{Time cost (in seconds) of computing Zen/TE-Score for ResNet-18/50 at resolution 224x224. The statistical error is within $5\%$. `time': time for computing Zen/TE-score for $N$ images, measured in seconds, averaged over 100 trials. `speed-up': speed-up rate of TE-Score v.s. Zen-Score.\\
 $\dagger$: The official implementation outputs Inf score for ResNet-18/50.}
 \label{tab:time-cost-zen-vs-TE}
\end{table}

\begin{figure}[!]
 \begin{center}
   \includegraphics[width=\linewidth]{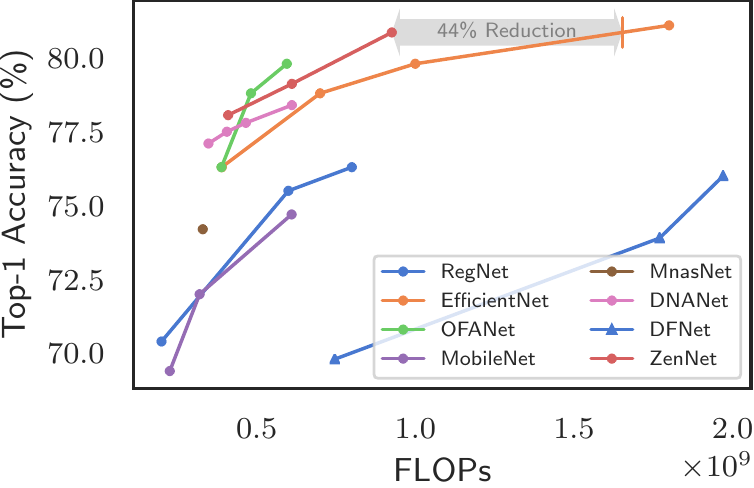}
   \caption{ZenNets optimized for FLOPs.}
   \label{fig:zennet-flops-vs-top1-acc}
 \end{center}
\end{figure}

\begin{table}[!]
 \begin{center}
   \begin{tabular}{@{}llll@{}}
     \toprule
     NAS    & Method  & Top-1 (\%) & GPU Day \\ \midrule
     AmoebaNet-A \citep{realRegularizedEvolutionImage2019}    & EA            & 74.5       & 3150$\dagger$     \\
     EcoNAS \citep{zhouEcoNASFindingProxies2020}         & EA            & 74.8       & 8       \\
     CARS-I \citep{yangCARSContinuousEvolution2020}         & EA            & 75.2       & 0.4     \\
     GeNet \citep{xieGeneticCNN2017}          & EA            & 72.1       & 17      \\
     DARTS \citep{liuDARTSDifferentiableArchitecture2019}          & GD            & 73.1       & 4       \\
     SNAS \citep{xieSNASStochasticNeural2018}           & GD            & 72.7       & 1.5     \\
     PC-DARTS \citep{xuPCDARTSPartialChannel2019}       & GD            & 75.8       & 3.8     \\
     ProxylessNAS \citep{cai_proxylessnas:_2019}   & GD            & 75.1       & 8.3     \\
     GDAS \citep{zhangOvercomingMultiModelForgetting2020}           & GD            & 74         & 0.8     \\
     FBNetV2-L1 \citep{wanFBNetV2DifferentiableNeural2020}     & GD            & 77.2       & 25      \\
     NASNet-A \citep{zophLearningTransferableArchitectures2018}       & RL            & 74         & 1800    \\
     Mnasnet-A \citep{tan_mnasnet:_2019}      & RL            & 75.2       & -       \\
     MetaQNN \citep{bakerDesigningNeuralNetwork2017}            & RL            & 77.4       & 96      \\
     PNAS \citep{liuProgressiveNeuralArchitecture2018}           & SMBO          & 74.2       & 224     \\
     SemiNAS \citep{luoSemiSupervisedNeuralArchitecture2020} & SSL & 76.5 & 4 \\
     TE-NAS \cite{chen2021neural} & ZS & 74.1 & 0.2 \\
     OFANet \citep{caiOnceforAllTrainOne2020}         & PS            & 80.1       & 51.6    \\
     EfficientNet-B7 \citep{tanEfficientNetRethinkingModel2019} & Scaling    & 84.4       & 3800$\ddagger$    \\ \midrule
     Zen-NAS    & ZS    & 83.6       & 0.5    \\ \bottomrule
     \end{tabular}
 \end{center}  
 \caption{NAS searching cost comparison. 
 'Top-1': top-1 accuracy on ImageNet-1k. 
 'Method': 'EA' is short for Evolutionary Algorithm; 'GD' is short for Gradient Descent; 'RL' is short for reinforcement Learning; 'ZS' is short for Zero-shot; 'SMBO', 'SSL', 'PS' and 'Scaling' are special searching methods/frameworks. 
  $\dagger$: Running on TPU; $\ddagger$: The cost is estimated by \citep{wanFBNetV2DifferentiableNeural2020};}
 \label{tab:nas-cost}
\end{table}


In each trial, the initial structure is a randomly selected small network which is guaranteed to satisfy the inference budget. The kernel size is searched in set $\{3,5,7\}$. Following conventional designs, the number of stages is three for CIFAR-10/CIFAR-100 and five for ImageNet.  The evolutionary population size is 256, number of evolutionary iterations $T=96,000$. The resolution  is 32x32 for CIFAR-10/CIFAR-100 and 224x224 for ImageNet.

\subsection{Zen-Score v.s. Other Zero-Shot Proxies}
\label{sec:Zen-Score-vs-SOTA-Zero-Shot-Proxies}

Following \cite{abdelfattahZeroCostProxiesLightweight2021,chen2021neural},  we compare Zen-Score to  five zero-shot proxies: FLOPs, gradient-norm (grad) of network parameters, synflow \cite{tanakaPruningNeuralNetworks2020}, TE-NAS score (TE-Score) \cite{chen2021neural} and NASWOT \cite{mellorNeuralArchitectureSearch2021a}. For each proxy, we replace Zen-Score by that proxy in Algorithm \ref{alg:zen-nas} and then run Algorithm \ref{alg:zen-nas} for $T=96,000$ iterations to ensure convergence. Since synflow is the smaller the better, we use its negative value in Algorithm \ref{alg:zen-nas}.  Following convention, we search for best network on CIFAR-10/CIFAR-100 within model size $N\leq1\,\mathrm{M}$. The convergence curves are plotted in Supplementary \ref{sec:additional-figures}. In these figures, all six scores improves monotonically along  iterations. 

After the above NAS step, we train the network of the best score for each proxy under the same training setting. To provide a random baseline, we randomly generate networks. The width of the layer varies in range $[4,512]$, and the depth of each stage varies in range $[1,10]$. If the network size is larger than $1\,\mathrm{M}$, we shrink its width by factor $0.75$ each time until it satisfies the budget. 32 random networks are generated and trained in total.

The top-1 accuracy is reported in Table \ref{tab:top1-acc-for-five-proxy}. Zen-Score significantly outperforms the other five proxies on both CIFAR-10 and CIFAR-100. TE-Score and NASWOT are the runner-up proxies with similar performance, followed by synflow. It is not surprise to see that naive proxies, such as FLOPs and gradient-norm, perform poorly, even worse than random search.

To compare the computational efficiency of Zen-Score and TE-score, we compute two scores for ResNet-18 and ResNet-50 at 224x224 resolution. The expected time cost is averaged over 100 trials. We find that averaging  Zen/TE-Score over $N=16$ random images is sufficient to reduce the statistical error below $5\%$. The results are reported in Table \ref{tab:time-cost-zen-vs-TE}. The computation of Zen-Score is $20\sim28$ times faster than TE-Score. 

We tried our best to benchmark NASWOT for ResNet-18/50 using the official code. However, the official code always outputs Inf for ResNet-18/50 at resolution 224. Despite of the Inf issue, Zen-Score is 3.3x times faster than NASWOT on ResNet-18 and 1.6x times faster on ResNet-50.

\subsection{Zen-NAS on ImageNet}
\label{sec:Zen-NAS-on-ImageNet}

We use Zen-NAS to search efficient network (ZenNet) on ImageNet.  We consider the following popular networks as baselines: (a) manually-designed networks, including ResNet \citep{heDeepResidualLearning2016}, DenseNet \citep{huangDenselyConnectedConvolutional2017}, ResNeSt \citep{zhangResNeStSplitAttentionNetworks2020}, MobileNet-V2 \citep{sandlerMobileNetV2InvertedResiduals2018}  (b) NAS-designed networks for fast inference on GPU, including OFANet-9ms/11ms \citep{caiOnceforAllTrainOne2020}, DFNet \citep{liPartialOrderPruning2019}, RegNet \citep{radosavovicDesigningNetworkDesign2020}; (c) NAS-designed networks optimized for FLOPs, including OFANet-389M/482M/595M \citep{caiOnceforAllTrainOne2020}, DNANet \citep{liBlockwiselySupervisedNeural2020}, EfficientNet \citep{tanEfficientNetRethinkingModel2019}, Mnasnet \citep{tan_mnasnet:_2019}.

Among these networks, EfficientNet is a popular baseline in NAS-related works. EfficientNet-B0/B1 are suitable for mobile device for their small FLOPs and model size. EfficientNet-B3$\sim$B7 are large models that are best to be deployed on a high-end GPU. Although EfficientNet is optimized for FLOPs, its inference speed on GPU is within top-tier ones. Many previous works compare to EfficientNet by inference speed on GPU \citep{zhangResNeStSplitAttentionNetworks2020,caiOnceforAllTrainOne2020,radosavovicDesigningNetworkDesign2020}.

\paragraph{Searching Low Latency Networks}
Following previous works \citep{caiOnceforAllTrainOne2020,liPartialOrderPruning2019,radosavovicDesigningNetworkDesign2020}, we use Zen-NAS to optimize network inference speed on NVIDIA V100 GPU. We use Search Space I in this experiment. The inference speed is tested at batch size 64, half precision (float16). We search for networks of inference latency within 0.1/0.2/0.3/0.5/0.8/1.2 milliseconds (ms) per image. For testing inference latency, we set batch-size=64 and do mini-batch inference 30 times. The averaged inference latency is recorded. The top-1 accuracy on ImageNet v.s. inference latency is plotted in Figure~\ref{fig:fig_latency_on_V100_pytorch_FP16}. Clearly, ZenNets outperform baseline models in both accuracy and inference speed by a large margin. The largest model ZenNet-1.2ms achieves $83.6\%$ top-1 accuracy which is between EfficientNet-B5 and B6. It is about $4.9$x faster than EfficientNet at the same accuracy level.

\paragraph{Searching Lightweight Networks}
Following previous works \cite{caiOnceforAllTrainOne2020,tanEfficientNetRethinkingModel2019}, we use Zen-NAS to search lightweight networks with small FLOPs. We use Search Space II in this experiment. We search for networks of computational cost within 400/600/900\,M FLOPs.  Similar to OFANet and EfficientNet, we add SE-blocks after convolutional layers. The top-1 accuracy v.s. FLOPs is plotted in Figure \ref{fig:zennet-flops-vs-top1-acc}. Again, ZenNets outperform most models by a large margin. ZenNet-900M-SE achieves $80.8\%$ top-1 accuracy which is comparable to EfficientNet-B3 with $43\%$ fewer FLOPs. The runner-up is OFANet whose efficiency is similar to ZenNet. 




\subsection{Searching Cost of Zen-NAS  v.s. SOTA}
\label{sec:Searching-Cost-of-Zen-NAS}

The major time cost of Zen-NAS is the computation of Zen-Score. The network latency is predicted by an in-house latency predictor whose time cost is nearly zero. According to  Table~\ref{tab:time-cost-zen-vs-TE}, the computation of Zen-Score for ResNet-50 only takes $0.15$ second. This means that scoring 96,000 networks similar to ResNet-50 only takes $4$ GPU hours, or $0.17$ GPU day.

We compare Zen-NAS searching cost to SOTA NAS methods in Table \ref{tab:nas-cost}. Since every NAS method uses different settings, it is difficult to make a fair comparison that everyone agrees with. Nevertheless, we only concern about the best model reported in each paper and the corresponding searching cost. This gives us a rough impression of the efficiency of these methods and their practical ability of designing high-performance models.

From Table \ref{tab:nas-cost}, for conventional NAS methods, it takes hundreds to thousands GPU days to find a good structure of accuracy better than $78.0\%$. Many one-shot methods are very fast. For most one-shot methods, the best accuracy is below  $80\%$. In comparison, Zen-NAS achieves $83.6\%$ top-1 accuracy within 0.5 GPU day. Among methods achieving above $80.0\%$ top-1 accuracy in Table~\ref{tab:nas-cost}, the searching speed of Zen-NAS is nearly 100 times faster than OFANet and 7800 times faster than EfficientNet. 
TE-NAS uses less GPU day than Zen-NAS in Table \ref{tab:nas-cost}. This does not conflict with Table \ref{tab:time-cost-zen-vs-TE} because the total number of networks evaluated by the two methods are different.

\section{Conclusion}

We proposed Zen-NAS, a zero-shot neural architecture search framework for designing high performance deep image recognition networks. Without optimizing network parameters, Zen-NAS ranks networks via network expressivity which can be numerically measured by Zen-Score. The searching speed of Zen-NAS is dramatically faster than previous SOTA methods. The ZenNets automatically designed by Zen-NAS are significantly more efficient in terms of inference latency, FLOPs and model size, in multiple recognition tasks. We wish the elegance of Zen-NAS will inspire more theoretical researches towards a deeper understanding of efficient network design.

{\small
\bibliographystyle{plain}
\bibliography{refs}
}

\clearpage \newpage
\appendix

\renewcommand{\topfraction}{.99}
\renewcommand{\floatpagefraction}{.99}%
\setlength{\dblfloatsep}{1ex}

\section{Datasets and Experiment Settings}
\label{sec:Datasets-and-Experiment-Settings}

\noindent \textbf{Dataset} CIFAR-10 has 50 thousand training images and 10 thousand testing images in 10 classes with resolution 32x32. CIFAR-100 has the same number of training/testing images but in 100 classes. ImageNet-1k has over 1.2 million training images and 50 thousand validation images in 1000 classes. We use the official training/validation split in our experiments.

\noindent \textbf{Augmentation} We use the following augmentations as in \citep{phamEfficientNeuralArchitecture2018a}: mix-up \citep{zhangMixupEmpiricalRisk2018}, label-smoothing \citep{szegedyRethinkingInceptionArchitecture2016}, random erasing \citep{zhongRandomErasingData2020}, random crop/resize/flip/lighting and AutoAugment \citep{cubukAutoAugmentLearningAugmentation2018}.

\noindent \textbf{Optimizer} For all experiments, we use SGD optimizer with momentum 0.9; weight decay 5e-4 for CIFAR-10/100, 4e-5 for ImageNet; initial learning rate $0.1$ with batch size 256; cosine learning rate decay \citep{loshchilovSGDRStochasticGradient2017}. We train models up to 1440 epochs in CIFAR-10/100, 480 epochs in ImageNet. Following previous works \citep{aguilarKnowledgeDistillationInternal2020,liBlockwiselySupervisedNeural2020,caiOnceforAllTrainOne2020}, we use EfficientNet-B3 as teacher networks when training ZenNets.

\section{Implementation}

Our code is implemented in PyTorch. The synflow implementation is available from \url{https://github.com/mohsaied/zero-cost-nas/blob/main/foresight/pruners/measures/synflow.py}. The official TE-NAS score implementation is available from \url{https://github.com/VITA-Group/TENAS/blob/main/lib/procedures}. The official NASWOT implementation is available from \url{https://github.com/BayesWatch/nas-without-training}.
Our searching and training code are released on \url{https://github.com/idstcv/ZenNAS}. 

\section{Additional Figures}
\label{sec:additional-figures}

\begin{figure}[!h]
 \centering
 \includegraphics[width=\linewidth]{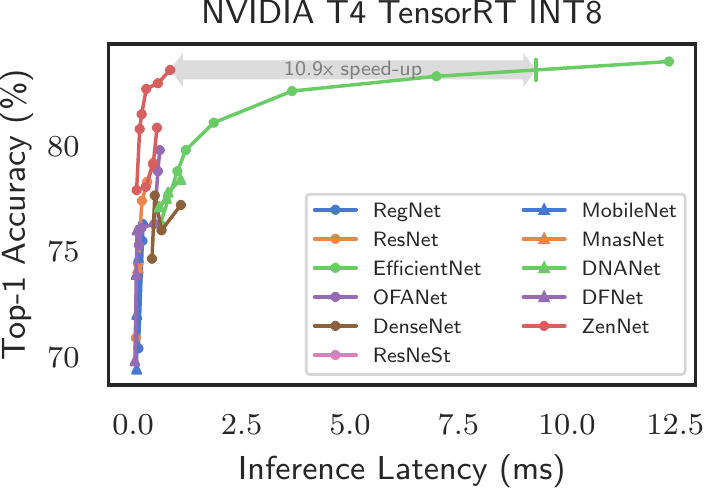}
 \caption{ZenNets top-1 accuracy on ImageNet-1k v.s. inference latency (milliseconds per image) on NVIDIA T4, TensorRT INT8, batch size $64$. ZenNet-0.8ms$\sim$1.2ms and ZenNet-400M-SE$\sim$900M-SE are plotted as two separated curves.} \label{fig:fig_plot_latency_on_T4_TRT_INT8}
\end{figure}

\begin{figure}[!h]
 \centering
 \includegraphics[width=\linewidth]{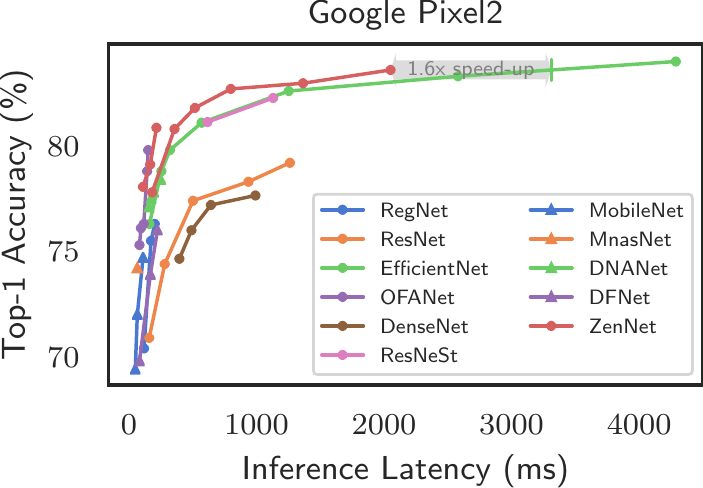}
 \caption{ZenNets top-1 accuracy on ImageNet-1k v.s. inference latency (milliseconds per image) on Google Pixel2, single image. ZenNet-0.8ms$\sim$1.2ms and ZenNet-400M-SE$\sim$900M-SE are plotted as two separated curves.} \label{fig:fig_plot_latency_on_Pixel2_MNN_INT8}
\end{figure}

We test the performance of ZenNets on devices other than NVIDIA V100 GPU. The two hardware platforms are considered. NVIDIA T4 is an industrial level GPU optimized for INT8 inference. All networks are exported to TensorRT engine at precision INT8 to benchmark their inference speed on T4. Google Pixel2 is a modern cell phone with moderate powerful mobile GPU. In Figure~\ref{fig:fig_plot_latency_on_T4_TRT_INT8} and Figure~\ref{fig:fig_plot_latency_on_Pixel2_MNN_INT8}, we report the inference speed of ZenNets on T4 and Pixel2 as well as several SOTA models. The best ZenNet-1.2ms is 10.9x times faster than EfficientNet on NVIDIA T4, 1.6x times faster on Pixel2. 

The evolutionary processes of optimizing zero-shot proxies are plotted in Figure \ref{fig:fig_cifar_ablation_EA_curve_zen}, \ref{fig:fig_cifar_ablation_EA_curve_flops}, \ref{fig:fig_cifar_ablation_EA_curve_grad_norm}, \ref{fig:fig_cifar_ablation_EA_curve_syncflow}, \ref{fig:fig_cifar_ablation_EA_curve_NASWOT}.

 \begin{figure}[!h]
   \begin{center}
     \includegraphics[width=0.9\linewidth]{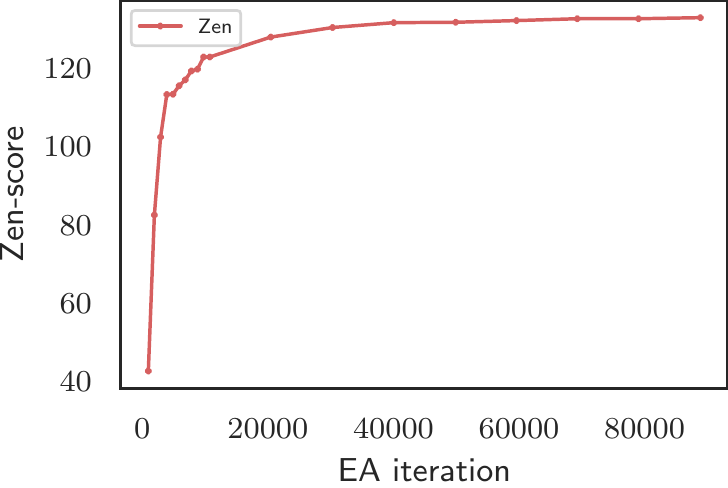}
   \end{center}
 \caption{NAS process for maximizing Zen-Score. x-axis: number of evolutionary iterations. y-axis: Largest Zen-Score in the current population.}
 \label{fig:fig_cifar_ablation_EA_curve_zen}
\end{figure}

\begin{figure}[!h]
   \begin{center}
     \includegraphics[width=0.9\linewidth]{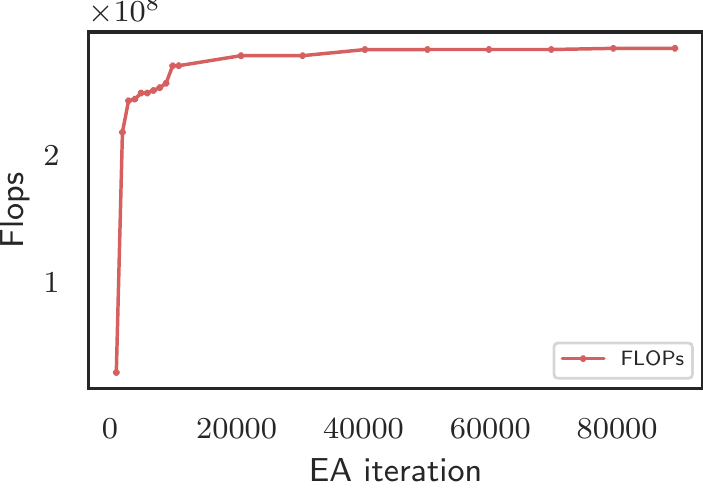}
   \end{center}
 \caption{NAS process for maximizing FLOPs. x-axis: number of evolutionary iterations. y-axis: Largest FLOPs in the current population.}
 \label{fig:fig_cifar_ablation_EA_curve_flops}
\end{figure}

\begin{figure}[!h]
   \begin{center}
     \includegraphics[width=0.9\linewidth]{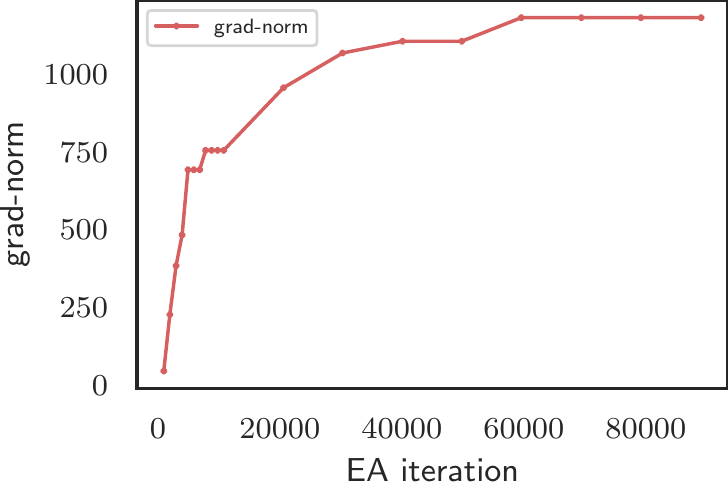}
   \end{center}
 \caption{NAS process for maximizing grad-norm. x-axis: number of evolutionary iterations. y-axis: Largest grad-norm in the current population.}
 \label{fig:fig_cifar_ablation_EA_curve_grad_norm}
\end{figure}

\begin{figure}[!h]
   \begin{center}
     \includegraphics[width=0.9\linewidth]{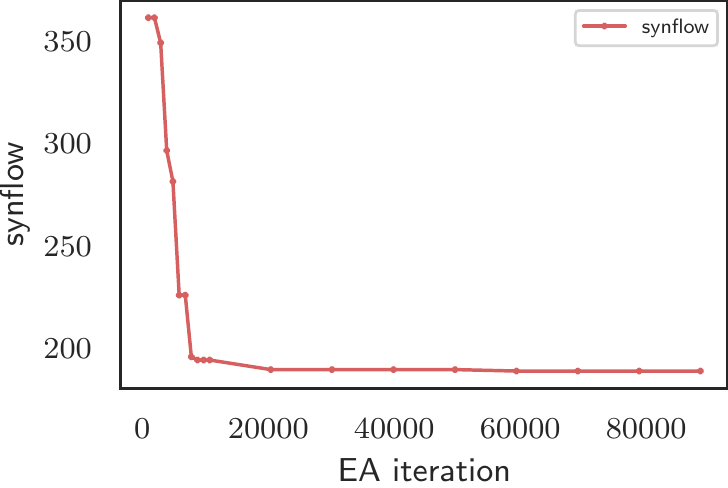}
   \end{center}
 \caption{NAS process for maximizing synflow. x-axis: number of evolutionary iterations. y-axis: Smallest synflow in the current population.}
 \label{fig:fig_cifar_ablation_EA_curve_syncflow}
\end{figure}

\begin{figure}[!h]
 \begin{center}
   \includegraphics[width=0.9\linewidth]{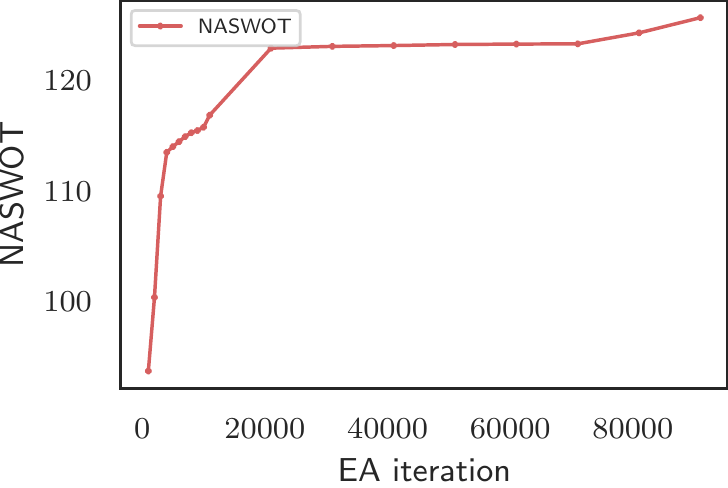}
 \end{center}
\caption{NAS process for maximizing NASWOT. x-axis: number of evolutionary iterations. y-axis: Largest NASWOT score in the current population.}
\label{fig:fig_cifar_ablation_EA_curve_NASWOT}
\end{figure}

\begin{figure}[!h]
   \begin{center}
     \includegraphics[width=0.9\linewidth]{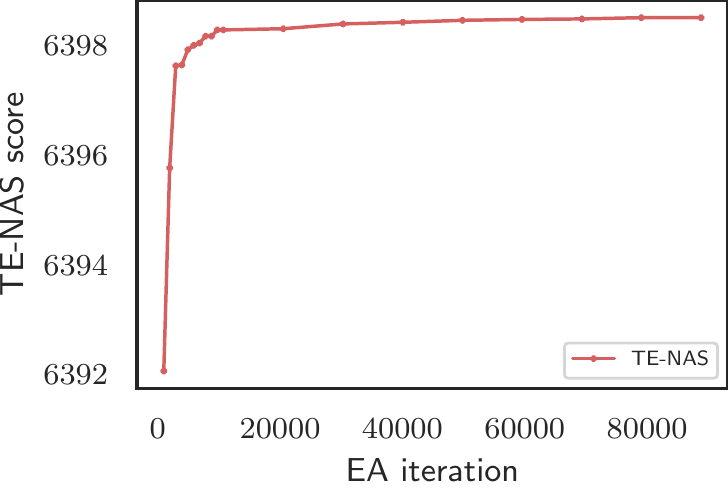}
   \end{center}
 \caption{NAS process for maximizing TE-NAS score. x-axis: number of evolutionary iterations. y-axis: Largest TE-NAS score in the current population. The NTK score in TE-NAS is the smaller the better. Therefore we use $R_N - \textrm{NTK}$ as TE-score in EA. This is slightly different from \cite{chen2021neural} where the rank of NTK is used as score.}
 \label{fig:fig_cifar_ablation_EA_curve_tenas}
\end{figure}

\section{Zen-NAS on CIFAR}
\label{sec:Zen-NAS-on-CIFAR}

Following previous works, we use Zen-NAS to optimize model size on CIFAR-10 and CIFAR-100 datasets. We use Search Space I in this experiment. We constrain the number of network parameters within \{1.0\,M, 2.0\,M\}. The resultant networks are labeled as ZenNet-1.0M/2.0M. Table~\ref{tab:zen-net-for-cifar}   summarized our results. We compare several popular NAS-designed models for CIFAR-10/CIFAR-100 in Figure~\ref{fig:fig_model_perf_CIFAR-10_CIFAR-100}, including AmoebaNet \citep{realRegularizedEvolutionImage2019}, DARTS \citep{liuDARTSDifferentiableArchitecture2019}, P-DARTS \citep{chenProgressiveDARTSBridging2019}, SNAS \citep{xieSNASStochasticNeural2018}, NASNet-A \citep{zophLearningTransferableArchitectures2018}, ENAS\citep{phamEfficientNeuralArchitecture2018a}, PNAS \citep{liuProgressiveNeuralArchitecture2018}, ProxylessNAS \citep{cai_proxylessnas:_2019}. ZenNets outperform baseline methods by $30\% \sim 50\%$ parameter reduction while achieving the same accuracies.

\begin{figure}[!h]
  \begin{center}
    \includegraphics[width=0.9\linewidth]{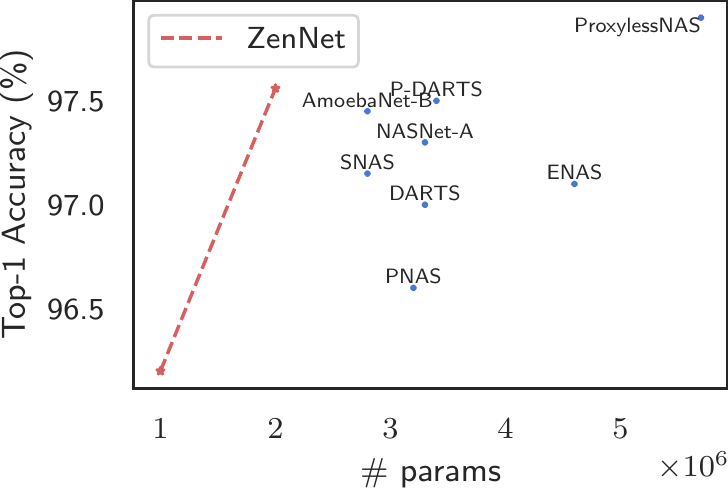} \\
    (a) CIFAR-10 \\      
    \includegraphics[width=0.9\linewidth]{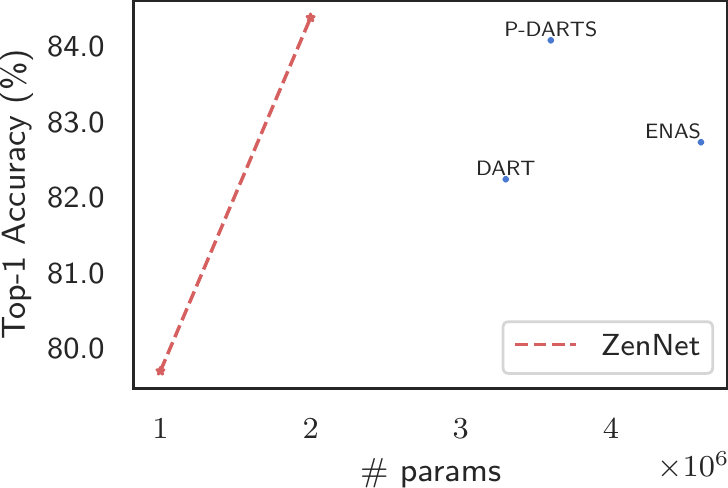} \\
    (b) CIFAR-100 \\
  \end{center}
\caption{ZenNet accuracy v.s. model size (\# params) on CIFAR-10 and CIFAR-100.} \label{fig:fig_model_perf_CIFAR-10_CIFAR-100}
\end{figure}

\begin{table}[!h]
\begin{center}
  \begin{tabular}{@{}lllll@{}}
    \toprule    
    model       & \# params    & FLOPs         & C10     & C100           \\ \midrule
    ZenNet-1.0M & 1.0\,M & 162\,M & 96.5\%    & 80.1\%               \\
    ZenNet-2.0M & 2.0\,M & 487\,M & 97.5\%      & 84.4\%            \\ \bottomrule
    \end{tabular}
\end{center}  
\caption{ZenNet-1.0M/2.0M on CIFAR-10 (C10) and CIFAR-100 (C100).}
\label{tab:zen-net-for-cifar}
\end{table}

\section{Zen-Scores and Accuracies of ResNets under Fair Training Setting}
\label{sec:zen-score-accuracy-of-resnets-under-fair-training}

\begin{table} [!h]
  \begin{center}
    \begin{tabular}{lccc}
      \toprule 
      Model & FLOPs & \# Params & Zen-Score\tabularnewline
      \midrule
      \midrule 
      ResNet-18 & 1.82G & 11.7M & 59.53\tabularnewline
      \midrule 
      ResNet-34 & 3.67G & 21.8M & 112.32\tabularnewline
      \midrule 
      ResNet-50 & 4.12G & 25.5M & 140.3\tabularnewline
      \midrule 
      ResNet-101 & 7.85G & 44.5M & 287.87\tabularnewline
      \midrule 
      ResNet-152 & 11.9G & 60.2M & 433.57\tabularnewline
      \bottomrule
    \end{tabular}    
  \end{center}  
  \caption{Zen-Scores of ResNets.}
  \label{tab:zen-score-of-resnets}
\end{table}

\begin{table} [!h]
  \begin{center}
    \begin{tabular}{lcc}
      \toprule 
      Model & Top-1 \citep{heDeepResidualLearning2016} & Top-1 (ours)\tabularnewline
      \midrule
      \midrule 
      ResNet-18 & 70.9\% & 72.1\%\tabularnewline
      \midrule 
      ResNet-34 & 74.4\% & 76.3\%\tabularnewline
      \midrule 
      ResNet-50 & 77.4\% & 79.0\%\tabularnewline
      \midrule 
      ResNet-101 & 78.3\% & 81.0\%\tabularnewline
      \midrule 
      ResNet-152 & 79.2\% & 82.3\%\tabularnewline
      \bottomrule
      \end{tabular}
  \end{center}  
  \caption{Top-1 accuracies of ResNets. Reported by \citep{heDeepResidualLearning2016} and using enhanced training methods we used in this paper.}
  \label{tab:top-1-resnets-he-vs-ours}
\end{table}

\begin{figure}[!h]
  \begin{center}
    \includegraphics{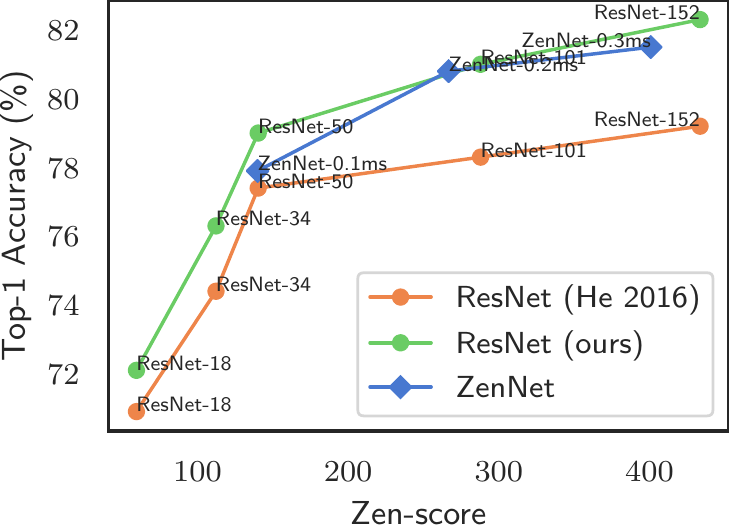}    
  \end{center}
  \caption{ResNet/ZenNet Zen-Score v.s. top-1 accuracy on ImageNet.}
  \label{fig:resnet-zen-score-vs-acc}
\end{figure}

ResNets are widely used in computer vision. It is interesting to understand the ResNets via Zen-Score analysis. We report the Zen-Scores of ResNets in Table~\ref{tab:zen-score-of-resnets}. In Figure~\ref{fig:resnet-zen-score-vs-acc}, we plot the Zen-Score against top-1 accuracy of ResNet and ZenNet on ImageNet. From the figure, it is clearly that even for the same model, the training method matters a lot. There is considerable performance gain of ResNets after using our enhanced training methods. The Zen-Scores positively correlate to the top-1 accuracies for both ResNet and ZenNets.

Next we show that the Zen-Scores is well-aligned with top-1 accuracies across different models. We consider two baselines in Table~\ref{tab:top-1-resnets-he-vs-ours}. The 2nd column reports the top-1 accuracies obtained in the ResNet original paper \citep{heDeepResidualLearning2016}. We found that these models are under-trained. We use enhanced training methods to train ResNets in the same way as we trained ZenNets. The corresponding top-1 accuracies are reported in the 3rd column.


\section{Effectiveness of Zen-Score}
\label{sec:Effectiveness-of-Zen-Score}

We show that Zen-Score effectively indicates the model accuracy during the evolutionary search. In the searching process of ZenNet-1.0M, we uniformly sample 16 structures from the evolutionary population. These structures have different number of channels and layers. Then the sampled structures are trained on CIFAR-10/CIFAR-100. The top-1 accuracy v.s. Zen-Score are plotted in Figure~\ref{fig:zen-score-vs-acc}. The Zen-Scores effectively indicates the network accuracies, especially in high-precision regimes.

\begin{figure}[!]
 \begin{minipage}{0.48\linewidth}
   \begin{center}
     \includegraphics[width=\linewidth]{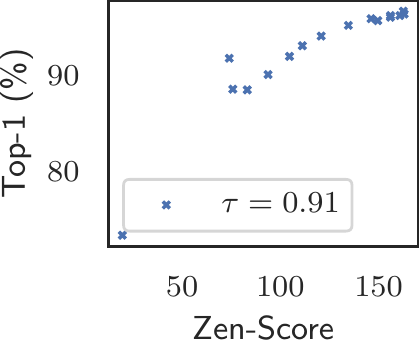}\\
     (a) CIFAR-10
   \end{center}
 \end{minipage}
 \hfil  
 \begin{minipage}{0.48\linewidth}
   \begin{center}
     \includegraphics[width=\linewidth]{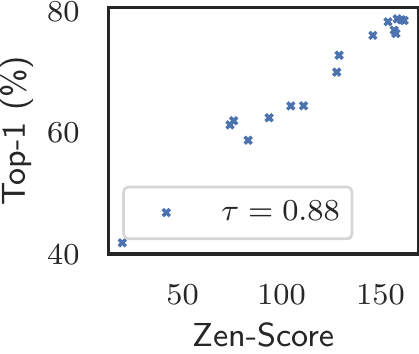}\\
     (b) CIFAR-100
   \end{center}
 \end{minipage}\\
 \caption{Zen-Score v.s. top-1 accuracy, 16 randomly sampled structures generated from ResNet-50, with Kendall's $\tau$-score between accuracy and Zen-Score.}
 \label{fig:zen-score-vs-acc}
\end{figure}

\section{FLOPs/Params/Latency of ZenNets in Table \ref{tab:top1-acc-for-five-proxy}}

\begin{center}
 \begin{tabular}{@{}lllll@{}}
    \toprule
    proxy     & params & FLOPs & latency &  \\ \midrule
    Zen-Score & 1.0M   & 170M  & 0.15ms  &  \\
    FLOPs     & 1.0M   & 285M  & 0.07ms  &  \\
    grad      & 0.2M   & 41M  & 0.14ms  &  \\
    synflow   & 1.0M   & 104M  & 0.11ms  &  \\
    TE-Score  & 1.0M   & 118M  & 0.08ms  &  \\
    NASWOT    & 1.0M   & 304M  & 0.25ms  &  \\
    Random    & 1.0M   & 110 M & 0.09ms  &  \\ \bottomrule
 \end{tabular}      
\end{center}

Latency is measured on NVIDIA V100 FP16 batch size 64. `grad' cannot find a model near params$\approx$1M.




\clearpage \newpage
\onecolumn

\section{Proof of Theorem \ref{thm:zen-score-equal-phi-score}}

\label{sec:Proof-of-Theorem-1}

We introduce a few more notations for our proof. Suppose the network
has $L$ convolutional layers. The $t$-th layer has $m_{t-1}$ input
channels and $m_{t}$ output channels. The convolutional kernel is
$\boldsymbol{\theta}_{t}\in\mathbb{R}^{m_{t}\times m_{-1t}\times k\times k}$.
The image resolution is $H\times W$. The mini-batch size is $B$.
The output feature map of the $t$-th layer is $\boldsymbol{x}_{t}$.
We use $\boldsymbol{x}_{t}^{(i,b,h,w)}$ to denote the pixel of $\boldsymbol{x}_{t}$
in the $i$-th channel, $b$-th image at cooridinate $(h,w)$. $\mathcal{N}(\mu,\sigma)$
denotes Gaussian distribution with mean $\mu$ and variance $\sigma^{2}$.
For random variables $z,a$ and a constant $b$, the notation $z=a\pm b$
means $|z-a|\leq b$. To avoid notation clutter, we use $C_{\log}^{1/\delta}(\cdot)$ to denote some logarithmic polynomial in $1/\delta$ and some other variables. Since the order of these variables in $C_{\log}^{1/\delta}(\cdot)$ is logarithmic, they do not alternate the polynomial order of our bounds.

The input image $\boldsymbol{x}_{0}$ are sampled from $\mathcal{N}(0,1)$.
In a vanilla network without BN layer, the feature map $\bar{\boldsymbol{x}}_{t}$
is generated by the following forward inference process:
\begin{align*}
\bar{\boldsymbol{x}}_{0}= & \boldsymbol{x}_{0}\\
\bar{\boldsymbol{x}}_{t}= & \left[\boldsymbol{\theta}_{t}*\bar{\boldsymbol{x}}_{t-1}\right]_{+}
\end{align*}
where $*$ is the convolutional operator, $[z]_{+}=\max(z,0)$. 

In Zen-Score computation, BN layer is inserted after every convolutional
operator. The forward inference now becomes:
\begin{align}
\boldsymbol{g}_{t}= & \boldsymbol{\theta}_{t}*\boldsymbol{x}_{t-1}\\{}
[\boldsymbol{\sigma}_{t}^{(i)}]^{2}= & \frac{1}{BHW}\sum_{b,h,w}[\boldsymbol{g}_{t}^{(i,b,h,w)}]^{2}\\
\bar{\sigma}_{t}^{2}= & \frac{1}{m_{t}}\sum_{i=1}^{m_{t}}[\boldsymbol{\sigma}_{t}^{(i)}]^{2}\\
\boldsymbol{x}_{t}^{(i)}= & \left[\frac{\boldsymbol{g}_{t}^{(i)}}{\boldsymbol{\sigma}_{t}^{(i)}}\right]_{+}=\frac{1}{\boldsymbol{\sigma}_{t}^{(i)}}[\boldsymbol{g}_{t}^{(i)}]_{+}\ .\label{eq:modified-bn-layer-definition}
\end{align}
Please note that in Eq. (\ref{eq:modified-bn-layer-definition}),
we use a modified BN layer instead of the standard BN, where we do
not subtract mean value in the normalization step. This will greatly
simply the proof. If the reader is concerned about this, it is straightforward
to replace all BN layers with our modified BN layers so that the computational
process exactly follows our proof. In practice, we did not observe
noticable difference by switching between two BN layers because the
mean value of mini-batch is very close to zero.

To show that the Zen-Score computed on BN-enabled network $f(\boldsymbol{x}_{0})=\boldsymbol{x}_{L}$
approximates the $\Phi$-score computed on BN-free network $\bar{f}(\boldsymbol{x}_{0})=\bar{\boldsymbol{x}}_{L}$,
we only need to prove 
\begin{equation}
(\prod_{t=1}^{L}\bar{\sigma}_{t})^{2}\mathbb{E}_{\theta}\|\boldsymbol{x}_{L}\|^{2}\approx\mathbb{E}_{\theta}\|\bar{\boldsymbol{x}}_{L}\|^{2}\ .\label{eq:final-prove-target}
\end{equation}
In deed, when Eq. (\ref{eq:final-prove-target}) holds true, by taking
gradient w.r.t. $\boldsymbol{x}$ on both side, the proof is then
completed. To prove Eq. (\ref{eq:final-prove-target}), we need the
following theorems and lemmas.

\subsection{Useful Theorems and Lemmas}

The first theorem is Bernstein's inequality. It can be found in many
statistical textbooks, such as \cite[Theorem 2.8.1]{vershynin_2018}.
\begin{thm}[Bernstein's inequality]
\label{thm:Bernstein-inequality} Let $x_{1},x_{2},\cdots,x_{N}$
be independent bounded random variables of mean zero, variance $\sigma$.
$|x_{i}|\leq K$ for all $i\in\{1,2,\cdots,N\}$. $\boldsymbol{a}=[a_{1},a_{2},\cdots,a_{N}]$
is a fixed $N$-dimensional vector. Then $\forall t\geq0$, 
\[
\mathbb{P}(\left|\sum_{i=1}^{N}a_{i}x_{i}\right|\geq t)\leq2\exp\left[-c\min\left(\frac{t^{2}}{\sigma^{2}\|\boldsymbol{a}\|_{2}^{2}},\frac{t}{K\|\boldsymbol{a}\|_{\infty}}\right)\right]\ .
\]

\end{thm}

A direct corollary gives the upper bound of sum of random variables.
\begin{cor}
\label{cor:bernstein-upper-bound} Under the same setting of Theorem
\ref{thm:Bernstein-inequality}, with probability at least $1-\delta$,
\[
\left|\sum_{i=1}^{N}a_{i}x_{i}\right|\leq C_{\log}^{1/\delta}(\cdot)\sigma\|\boldsymbol{a}\|_{2}\ .
\]
\end{cor}

\begin{proof}
Let 
\begin{align*}
\delta & \triangleq2\exp\left[-c\min\left(\frac{t^{2}}{\sigma^{2}\|\boldsymbol{a}\|_{2}^{2}},\frac{t}{K\|\boldsymbol{a}\|_{\infty}}\right)\right]\\
& =\max\left\{ 2\exp\left[-c\frac{t^{2}}{\sigma^{2}\|\boldsymbol{a}\|_{2}^{2}}\right],2\exp\left[-c\frac{t}{K\|\boldsymbol{a}\|_{\infty}}\right]\right\} \ .
\end{align*}
That is,
\begin{align*}
& \delta\geq2\exp\left[-c\frac{t^{2}}{\sigma^{2}\|\boldsymbol{a}\|_{2}^{2}}\right]\\
\Leftrightarrow & t\leq\sqrt{\frac{1}{c}\log(2/\delta)}\sigma\|\boldsymbol{a}\|_{2}=C_{\log}^{1/\delta}(\cdot)\sigma\|\boldsymbol{a}\|_{2}\ ,
\end{align*}
and
\begin{align*}
& \delta\geq2\exp\left[-c\frac{t}{K\|\boldsymbol{a}\|_{\infty}}\right]\\
\Leftrightarrow & t\leq\frac{1}{c}\log(2/\delta)K\|\boldsymbol{a}\|_{\infty}=C_{\log}^{1/\delta}(\cdot)K\|\boldsymbol{a}\|_{\infty}\ .
\end{align*}
Therefore, with probability at least $1-\delta$,
\begin{align*}
\left|\sum_{i=1}^{N}a_{i}x_{i}\right|\leq & \min\{C_{\log}^{1/\delta}(\cdot)\sigma\|\boldsymbol{a}\|_{2},C_{\log}^{1/\delta}(\cdot)K\|\boldsymbol{a}\|_{\infty}\}\\
\leq & C_{\log}^{1/\delta}(\cdot)\min\{\sigma\|\boldsymbol{a}\|_{2},K\|\boldsymbol{a}\|_{\infty}\}\ .
\end{align*}

That is,
\[
\left|\sum_{i=1}^{N}a_{i}x_{i}\right|\leq C_{\log}^{1/\delta}(\cdot)\sigma\|\boldsymbol{a}\|_{2}\ .
\]
\end{proof}
When the random variables are sampled from Gaussian distribution,
it is more convenient to use the following tighter bound.
\begin{thm}
\label{thm:sum-gaussian-tail} Let $x_{1},x_{2},\cdots,x_{N}$ be
sampled from $\mathcal{N}(0,\sigma)$, $\boldsymbol{a}\in\mathbb{R}^{N}$
be a fixed a vector. Then $\forall t\geq0$, 
\[
\mathbb{P}(\left|\sum_{i=1}^{N}a_{i}x_{i}\right|>t)\leq\exp\left[-\frac{t^{2}}{2\sigma^{2}\|a\|_{2}^{2}}\right]\ .
\]
\end{thm}

\begin{cor}
\label{cor:sum-gaussian-upper-bound} With probability at least $1-\delta$,
\[
\left|\sum_{i=1}^{N}a_{i}x_{i}\right|\leq\sqrt{2\log(1/\delta)}\sigma\|\boldsymbol{a}\|_{2}=C_{\log}^{1/\delta}(\cdot)\sigma\|\boldsymbol{a}\|_{2}\ .
\]

\end{cor}

The proof is simple since the sum of Gaussian random variables is
still Gaussian random variables.

The following two lemmas are critical in our lower bound analysis.
The proof is straightforward once using the symmetric property of
random variable distribution.
\begin{lem}
\label{lem:square-norm-of-sym-var} Suppose $x\in\mathbb{R}$ is a
mean zero, variance $\sigma^{2}$ random variable with symmetric distribution.
Then $\mathbb{E}[x]_{+}^{2}=4\sigma^{2}/4$.
\end{lem}

{}
\begin{lem}
\label{lem:expect-theta-norm-exchange} Suppose $\theta_{i}\sim\mathcal{N}(0,1)$.
$\|\boldsymbol{x}\|=\|\boldsymbol{y}\|$ are two fixed vectors. Then
\[
\mathbb{E}_{\theta}[\sum_{i}\theta_{i}x_{i}]_{+}^{2}=\frac{1}{2}\mathbb{E}_{\theta}[\sum_{i}\theta_{i}x_{i}]^{2}=\mathbb{E}_{\theta}[\sum_{i}\theta_{i}y_{i}]_{+}^{2}\ .
\]
\end{lem}

\subsection{Proof of Eq. (\ref{eq:final-prove-target})}

Since $\boldsymbol{x}_{0}\sim\mathcal{N}(0,1)$, with probability
at least $1-\delta$, $\|\boldsymbol{x}_{0}\|_{\infty}\leq C_{\log}^{1/\delta}(\cdot)\triangleq K_{0}$
for some constant $K_{0}$. Now suppose at layer $t$, $\|\boldsymbol{x}_{t-1}\|_{\infty}\leq K_{t-1}$.
The following lemma shows that after convolution, $\|\boldsymbol{g}_{t}\|_{\infty}$
is also bounded with high probability.
\begin{lem}
\label{lem:norm-bound-of-conv} Let $\boldsymbol{\theta}^{(i,b,h,w)}\sim\mathcal{N}(0,1)$,
$\boldsymbol{\theta}_{t}\in\mathbb{R}^{m_{t}\times m_{t-1}\times k\times k}$.
For fixed $\boldsymbol{x}_{t-1}\in\mathbb{R}^{m_{t-1}\times B\times H\times W}$,
$\boldsymbol{g}_{t}\triangleq\boldsymbol{\theta}_{t}*\boldsymbol{x}_{t-1}$.
Then with probability at least $1-\delta$,
\[
\|\boldsymbol{g}_{t}\|_{\infty}\leq C_{\log}^{1/\delta}(\cdot)^{2}k\sqrt{m_{t-1}}K_{t-1}\ .
\]
\end{lem}

\begin{proof}
Let us consider $\boldsymbol{g}_{t}^{(j,b,\alpha,\beta)}$, that is,
the $j$-th channel, $b$-th image, at pixel $(\alpha,\beta)$. For
any $1\leq j\leq m_{t}$, $1\leq\alpha\leq H$, $1\leq\beta\leq W$,
\begin{align*}
\boldsymbol{g}_{t}^{(j,b,\alpha,\beta)}= & \sum_{i=1}^{m_{-1t}}\sum_{p=-\frac{k-1}{2}}^{\frac{k-1}{2}}\sum_{q=-\frac{k-1}{2}}^{\frac{k-1}{2}}\boldsymbol{\theta}_{t}^{(j,i,p,q)}\boldsymbol{x}_{t-1}^{(i,b,\alpha+p,\beta+p)}
\end{align*}
Clearly,
\[
\mathbb{E}_{\theta}\boldsymbol{g}_{t}^{(j,b,\alpha,\beta)}=0\ .
\]
According to Corollary \ref{cor:sum-gaussian-upper-bound}, 

\begin{align*}
|\boldsymbol{g}_{t}^{(j,b,\alpha,\beta)}| & \leq C_{\log}^{1/\delta}(\cdot)C_{\log}^{1/\delta}(\cdot)K_{t-1}\sqrt{m_{t-1}}k\\
& \leq C_{\log}^{1/\delta}(\cdot)^{2}k\sqrt{m_{t-1}}K_{t-1}\ .
\end{align*}
\end{proof}
The variance of $\boldsymbol{g}_{t}$ is bounded with high probability
too.
\begin{lem}
\label{lem:var-gt} With probability at least $1-\delta$,
\begin{align*}
\mathbb{E}[\boldsymbol{g}_{t}^{(j,b,\alpha,\beta)}]^{2}= & \sigma_{t}^{*}\pm C_{\log}^{1/\delta}(\cdot)k\sqrt{m_{t-1}}K_{t-1}\\
\sigma_{t}^{*2}\triangleq & \frac{1}{4}m_{t-1}k^{2}\ .
\end{align*}
\end{lem}

\begin{proof}
By definition,
\begin{align*}
\boldsymbol{g}_{t}^{(j,b,\alpha,\beta)}= & \sum_{i=1}^{m_{-1t}}\sum_{p=-\frac{k-1}{2}}^{\frac{k-1}{2}}\sum_{q=-\frac{k-1}{2}}^{\frac{k-1}{2}}\boldsymbol{\theta}_{t}^{(j,i,p,q)}\boldsymbol{x}_{t-1}^{(i,b,\alpha+p,\beta+p)}
\end{align*}
Clearly, $\boldsymbol{g}_{t}^{(j,b,\alpha,\beta)}$ is Gaussian random
variable with zero-mean.
\begin{align*}
\mathbb{E}[\boldsymbol{g}_{t}^{(j,b,\alpha,\beta)}]^{2}= & \sum_{i=1}^{m_{t-1}}\sum_{p=-\frac{k-1}{2}}^{\frac{k-1}{2}}\sum_{q=-\frac{k-1}{2}}^{\frac{k-1}{2}}[\boldsymbol{x}_{t-1}^{(i,b,\alpha+p,\beta+p)}]^{2}\ .
\end{align*}
By Lemma \ref{lem:square-norm-of-sym-var}, 
\begin{align*}
\mathbb{E}[\boldsymbol{x}_{t-1}^{(i,b,\alpha+p,\beta+p)}]^{2}= & \frac{1}{4}\ .
\end{align*}
Therefore,
\begin{align*}
& |\mathbb{E}[\boldsymbol{g}_{t}^{(j,b,\alpha,\beta)}]^{2}-\frac{1}{4}m_{t-1}k^{2}|\\
\leq & C_{\log}^{1/\delta}(\cdot)k\sqrt{m_{t-1}}K_{t-1}\ .
\end{align*}
Define $\sigma_{t}^{*2}\triangleq\frac{1}{4}m_{t-1}k^{2}$, the proof
is completed.
\end{proof}

Next we show that both $\sigma_{t}^{(i)}$ and $\bar{\sigma}_{t}$
concentrate around $\sigma^{*}$.
\begin{lem}
\label{lem:sigma-i-concentration} With probability
$1-\delta$, 
\begin{align*}
[\boldsymbol{\sigma}_{t}^{(i)}]^{2}= & (1\pm\epsilon_{t})[\sigma_{t}^{*}]^{2}\\
\bar{\sigma}_{t}= & (1\pm\frac{\epsilon_{t}}{\sqrt{m_{t}}})[\sigma_{t}^{*}]^{2}
\end{align*}
where
\[
\epsilon_{t}\triangleq4C_{\log}^{1/\delta}(\cdot)^{5}\frac{1}{\sqrt{BHW}}K_{t-1}^{2}
\]
\end{lem}

\begin{proof}
Directly apply Corollary \ref{cor:bernstein-upper-bound},
\begin{align*}
[\boldsymbol{\sigma}_{t}^{(i)}]^{2}= & \mathbb{E}[\boldsymbol{g}_{t}^{(j,b,\alpha,\beta)}]^{2}\pm C_{\log}^{1/\delta}(\cdot)\frac{1}{\sqrt{BHW}}\max\{[\boldsymbol{g}_{t}^{(j,b,\alpha,\beta)}]^{2}\}\\
= & \mathbb{E}[\boldsymbol{g}_{t}^{(j,b,\alpha,\beta)}]^{2}\pm C_{\log}^{1/\delta}(\cdot)\frac{1}{\sqrt{BHW}}C_{\log}^{1/\delta}(\cdot)^{4}m_{t-1}k^{2}K_{t-1}^{2}\\
= & [\sigma_{t}^{*}]^{2}\pm\frac{1}{\sqrt{BHW}}C_{\log}^{1/\delta}(\cdot)^{5}m_{t-1}k^{2}K_{t-1}^{2}\ .
\end{align*}

Similary,
\begin{align*}
\bar{\sigma}_{t}^{2}= & [\sigma_{t}^{*}]^{2}\pm\frac{1}{\sqrt{m_{t}BHW}}C_{\log}^{1/\delta}(\cdot)^{5}m_{t-1}k^{2}K_{t-1}^{2}\ .
\end{align*}

Define
\begin{align*}
\epsilon_{t}\triangleq & \frac{1}{[\sigma_{t}^{*}]^{2}}\frac{1}{\sqrt{BHW}}C_{\log}^{1/\delta}(\cdot)^{5}m_{t-1}k^{2}K_{t-1}^{2}\\
= & \frac{4}{m_{t-1}k^{2}}C_{\log}^{1/\delta}(\cdot)^{5}\frac{1}{\sqrt{BHW}}m_{t-1}k^{2}K_{t-1}^{2}\\
= & 4C_{\log}^{1/\delta}(\cdot)^{5}\frac{1}{\sqrt{BHW}}K_{t-1}^{2}
\end{align*}
Then we have
\begin{align*}
[\boldsymbol{\sigma}_{t}^{(i)}]^{2}= & (1\pm\epsilon_{t})[\sigma_{t}^{*}]^{2}\\
\bar{\sigma}_{t}= & (1\pm\frac{\epsilon_{t}}{\sqrt{m_{t}}})[\sigma_{t}^{*}]^{2}
\end{align*}
\end{proof}
Next is our main lemma.
\begin{lem}
\label{lem:sigma-xtp1-equal-gt-norm}
Under the same setting of Lemma \ref{lem:sigma-i-concentration}, with probability $1-\delta$,
\[
(\sigma_{t}^{*})^{2}\|\boldsymbol{x}_{t}\|^{2}=\frac{1}{1\pm\epsilon_{t}}\|\left[\boldsymbol{g}_{t}\right]_{+}\|^{2}\ .
\]
\end{lem}

\begin{proof}
By definition,
\begin{align*}
\|\boldsymbol{x}_{t}\|^{2}= & \sum_{i}\left[\frac{1}{\boldsymbol{\sigma}_{t}^{(i)}}\right]^{2}\left[\boldsymbol{g}_{t}^{(i)}\right]_{+}^{2}\\
= & \sum_{i}\left[\frac{1}{(1\pm\epsilon_{t})\sigma_{t}^{*}}\right]^{2}\left[\boldsymbol{g}_{t}^{(i)}\right]_{+}^{2}
\end{align*}
Then
\begin{align*}
& \frac{(\sigma_{t}^{*})^{2}\|\boldsymbol{x}_{t}\|^{2}}{\sum_{i}\left[\boldsymbol{g}_{t}^{(i)}\right]_{+}^{2}}\\
= & \frac{1}{\sum_{i}\left[\boldsymbol{g}_{t}^{(i)}\right]_{+}^{2}}\sum_{i}\left[\frac{\sigma_{t}^{*}}{\boldsymbol{\sigma}_{t}^{(i)}}\right]^{2}\left[\boldsymbol{g}_{t}^{(i)}\right]_{+}^{2}
\end{align*}
By Lemma \ref{lem:sigma-i-concentration}, we have
\begin{align*}
\frac{1}{1+\epsilon_{t}}\leq\frac{\sigma_{t}^{*}}{\boldsymbol{\sigma}_{t}^{(i)}}\leq & \frac{1}{1-\epsilon_{t}}
\end{align*}
\end{proof}
Finally, we inductively bound $|\boldsymbol{x}_{t}^{(i,b,h,w)}|$.
\begin{lem}
\label{lem:upper-bound-x-t} With probability at least $1-\delta$,
\begin{align*}
& |\boldsymbol{x}_{t}^{(i,b,h,w)}|\leq\frac{C_{\log}^{1/\delta}(\cdot)^{2}}{\sqrt{(1-\epsilon_{t})}}K_{t-1}\\
& K_{t}\leq C_{\log}^{1/\delta}(\cdot)^{2t}\prod_{j=1}^{t}(1-\epsilon_{j})^{-1/2}K_{0}\ .
\end{align*}
\end{lem}

\begin{proof}
By definition,
\[
\boldsymbol{x}_{t}^{(i,b,h,w)}=\frac{1}{\boldsymbol{\sigma}_{t}^{(i)}}[\boldsymbol{g}_{t}^{(i,b,h,w)}]_{+}
\]

From Lemma \ref{lem:norm-bound-of-conv},
\[
[\boldsymbol{g}_{t}^{(i,b,h,w)}]_{+}\leq C_{\log}^{1/\delta}(\cdot)^{2}K_{t-1}\sqrt{m_{t-1}}k
\]

From Lemma \ref{lem:sigma-i-concentration}, 
\begin{align*}
[\boldsymbol{\sigma}_{t}^{(i)}]^{2} & =(1\pm\epsilon_{t})[\sigma_{t}^{*}]^{2}\\
& =\frac{1}{4}(1\pm\epsilon_{t})m_{t-1}k^{2}
\end{align*}

Then
\begin{align*}
|\boldsymbol{x}_{t}^{(i,b,h,w)}|\leq & \frac{C_{\log}^{1/\delta}(\cdot)^{2}K_{t-1}\sqrt{m_{t-1}}k}{\sqrt{\frac{1}{4}(1\pm\epsilon_{t})m_{t-1}k^{2}}}\\
\leq & \frac{C_{\log}^{1/\delta}(\cdot)^{2}K_{t-1}\sqrt{m_{t-1}}k}{\sqrt{\frac{1}{4}(1-\epsilon_{t})m_{t-1}k^{2}}}\\
\leq & 2\frac{C_{\log}^{1/\delta}(\cdot)^{2}K_{t-1}}{\sqrt{(1-\epsilon_{t})}}\\
\rightarrow & \frac{C_{\log}^{1/\delta}(\cdot)^{2}K_{t-1}}{\sqrt{(1-\epsilon_{t})}}\quad\textrm{absorb \ensuremath{2} into \ensuremath{C_{\log}^{1/\delta}(\cdot)}}
\end{align*}
Therefore, 
\begin{align*}
& K_{t}\triangleq\frac{C_{\log}^{1/\delta}(\cdot)^{2}K_{t-1}}{\sqrt{(1-\epsilon_{t})}}\\
\Rightarrow & K_{t}=C_{\log}^{1/\delta}(\cdot)^{2t}\prod_{j=1}^{t}(1-\epsilon_{j})^{-1/2}K_{0}
\end{align*}
\end{proof}
Combining all above together, we are now ready to prove Eq. (\ref{eq:final-prove-target}).

Denote $z_{0}=1$. It is trivial to see that $z_{0}\|\boldsymbol{x}_{0}\|^{2}=z_{0}\|\bar{\boldsymbol{x}}_{t}\|^{2}$.
By induction, suppose at layer $t$, we already have $z_{t-1}\|\boldsymbol{x}_{t-1}\|^{2}=\|\bar{\boldsymbol{x}}_{t-1}\|^{2}$.
Using Lemma \ref{lem:expect-theta-norm-exchange},
\begin{align*}
\mathbb{E}_{\theta}\|\bar{\boldsymbol{x}}_{t}\|^{2}= & \mathbb{E}_{\theta}\|\left[\boldsymbol{\theta}_{t}*\bar{\boldsymbol{x}}_{t-1}\right]_{+}\|^{2}\\
= & \mathbb{E}_{\theta}\|\left[\boldsymbol{\theta}_{t}*z_{t-1}\boldsymbol{x}_{t-1}\right]_{+}\|^{2}\\
= & z_{t-1}\mathbb{E}_{\theta}\|\left[\boldsymbol{\theta}_{t}*\boldsymbol{x}_{t-1}\right]_{+}\|^{2}\\
= & z_{t-1}\mathbb{E}_{\theta}\|\left[\boldsymbol{g}_{t}\right]_{+}\|^{2}
\end{align*}

On the other hand, from Lemma \ref{lem:sigma-xtp1-equal-gt-norm},
\begin{align*}
\bar{\sigma}_{t}^{2}z_{t-1}\|\boldsymbol{x}_{t}\|^{2}= & z_{t-1}\frac{\bar{\sigma}_{t}^{2}}{(\sigma_{t}^{*})^{2}}(\sigma_{t}^{*})^{2}\|\boldsymbol{x}_{t}\|^{2}\\
= & z_{t-1}(1\pm\frac{\epsilon_{t}}{\sqrt{m_{t}}})(\sigma_{t}^{*})^{2}\|\boldsymbol{x}_{t}\|^{2}\quad\textrm{Lemma [lem:sigma-i-concentration]}\\
= & z_{t-1}(1\pm\frac{\epsilon_{t}}{\sqrt{m_{t}}})\frac{1}{1\pm\epsilon_{t}}\|\left[\boldsymbol{g}_{t}\right]_{+}\|^{2}\ .
\end{align*}
Therefore,
\[
\mathbb{E}_{\theta}\{\bar{\sigma}_{t}^{2}z_{t-1}\|\boldsymbol{x}_{t}\|^{2}\}=(1\pm\frac{\epsilon_{t}}{\sqrt{m_{t}}})\frac{1}{1\pm\epsilon_{t}}\mathbb{E}_{\theta}\|\bar{\boldsymbol{x}}_{t}\|^{2}
\]
By taking 
\begin{align*}
z_{t} & \triangleq\bar{\sigma}_{t}^{2}z_{t-1}/[(1\pm\frac{\epsilon_{t}}{\sqrt{m_{t}}})\frac{1}{1\pm\epsilon_{t}}]\ ,
\end{align*}
we complete the induction of $z_{t}\|\boldsymbol{x}_{t}\|^{2}=\|\bar{\boldsymbol{x}}_{t}\|^{2}$
for all $t$.

Chaining $t=\{1,2,\cdots,L\}$, we get
\begin{align*}
\mathbb{E}_{\theta}\{(\prod_{t=1}^{L}\bar{\sigma}_{t}^{2})\|\boldsymbol{x}_{L}\|^{2}\}=\prod_{t=1}^{L} & \left[(1\pm\frac{\epsilon_{t}}{\sqrt{m_{t}}})\frac{1}{1\pm\epsilon_{t}}\right]\mathbb{E}_{\theta}\|\bar{\boldsymbol{x}}_{L}\|^{2}\ ,
\end{align*}
where
\begin{align*}
\epsilon_{t}\triangleq & 4C_{\log}^{1/\delta}(\cdot)^{5}\frac{1}{\sqrt{BHW}}K_{t-1}^{2}\\
K_{t}\triangleq & C_{\log}^{1/\delta}(\cdot)^{2t}\prod_{j=1}^{t}(1-\epsilon_{j})^{-1/2}K_{0}\ .
\end{align*}

Finally, integrate everything together, we have proved that, with
probability at least $1-\delta$, 
\begin{align*}
(\prod_{t=1}^{L}\bar{\sigma}_{t}^{2})\mathbb{E}_{\theta}\{\|\boldsymbol{x}_{L}\|^{2}\} & =\prod_{t=1}^{L}\left[(1\pm\frac{\epsilon_{t}}{\sqrt{m_{t}}})\frac{1}{1\pm\epsilon_{t}}\right]\mathbb{E}_{\theta}\|\bar{\boldsymbol{x}}_{L}\|^{2}\ .
\end{align*}
That is,
\begin{align*}
& \prod_{t=1}^{L}\left[(1-\frac{\epsilon_{t}}{\sqrt{m_{t}}})\frac{1}{1+\epsilon}\right]\leq\\
& \qquad\frac{(\prod_{t=1}^{L}\bar{\sigma}_{t}^{2})\mathbb{E}_{\theta}\{\|\boldsymbol{x}_{L}\|^{2}\}}{\mathbb{E}_{\theta}\|\bar{\boldsymbol{x}}_{L}\|^{2}}\\
& \qquad\qquad\leq\prod_{t=1}^{L}\left[(1+\frac{\epsilon_{t}}{\sqrt{m_{t}}})\frac{1}{1-\epsilon_{t}}\right]\ .
\end{align*}

To further simply the above results, we consider the asymptotic case
where $BHW$ is large enough. Then $\epsilon_{t}$ will be a small
number. By first order approximation of binomial expansion, $(1+\epsilon)^{L}\approx1+L\epsilon+\mathcal{O}(\epsilon^{2})$.
To see that $\epsilon_{t}$ is bounded by a small constant, we denote
$\gamma_{t}\triangleq\max_{j\in[1,t]}\epsilon_{j}$. Then 
\begin{align*}
K_{t} \leq & \mathcal{O}[(1+\frac{t-1}{2}\gamma_{t-1})K_{0}]
\end{align*}
\begin{equation}
\gamma_{t} \leq \mathcal{O}\{\frac{K_{0}}{\sqrt{BHW}}[(1+\frac{(t-1)}{2}\gamma_{t-1}]\}\ .\label{eq:recursive-inequality-gamma-t}
\end{equation}
By the recursive equation Eq. (\ref{eq:recursive-inequality-gamma-t}),
when $\gamma_{t-1} \leq \frac{2}{L-1}$, 
\begin{align*}
\gamma_{t} & \leq \mathcal{O}\{\frac{K_{0}}{\sqrt{BHW}}[(1+\frac{(t-1)}{2}\gamma_{t-1}]\}\\
& \leq\mathcal{O}\{\frac{2K_{0}}{\sqrt{BHW}}\}\ .
\end{align*}
Therefore, by taking $\frac{2K_{0}}{\sqrt{BHW}}\leq\frac{2}{L}$,
that is $BHW\geq\mathcal{O}\{L^{2}K_{0}^{2}\}$, we have
\[
\epsilon=\max\epsilon_{t}\leq\mathcal{O}\{\frac{2K_{0}}{\sqrt{BHW}}\}
\]
to be a small number. 

When $\epsilon$ is a small number, 
\begin{align*}
\frac{(\prod_{t=1}^{L}\bar{\sigma}_{t}^{2})\mathbb{E}_{\theta}\{\|\boldsymbol{x}_{L}\|^{2}\}}{\mathbb{E}_{\theta}\|\bar{\boldsymbol{x}}_{L}\|^{2}}\leq & \prod_{t=1}^{L}\left[(1+\frac{\epsilon_{t}}{\sqrt{m_{t}}})\frac{1}{1-\epsilon_{t}}\right]\\
\leq & (1+\epsilon)^{L}(1-\epsilon_{t})^{-L}\\
\approx & (1+L\epsilon)^{2}\ .
\end{align*}
Similarly,
\[
\frac{(\prod_{t=1}^{L}\bar{\sigma}_{t}^{2})\mathbb{E}_{\theta}\{\|\boldsymbol{x}_{L}\|^{2}\}}{\mathbb{E}_{\theta}\|\bar{\boldsymbol{x}}_{L}\|^{2}}\geq(1-L\epsilon)^{2}\ .
\]

\clearpage \newpage

\section{One Big Table of Networks on ImageNet}
\label{sec:one-big-table}

\begin{longtable}[c]{llllllll}
  \toprule 
\multirow{2}{*}{model} & \multirow{2}{*}{resolution} & \multirow{2}{*}{\# params} & \multirow{2}{*}{FLOPs} & \multirow{2}{*}{Top-1 Acc} & \multicolumn{3}{c}{latency(ms)}\tabularnewline
\cmidrule{6-8} \cmidrule{7-8} \cmidrule{8-8} 
 &  &  &  &  & V100 & T4 & Pixel2\tabularnewline
\midrule 
RegNetY-200MF & 224 & 3.2\,M & 200\,M & 70.4\% & 0.22 & 0.12 & 118.17\tabularnewline
\midrule 
RegNetY-400MF & 224 & 4.3\,M & 400\,M & 74.1\% & 0.44 & 0.17 & 181.09\tabularnewline
\midrule 
RegNetY-600MF & 224 & 6.1\,M & 600\,M & 75.5\% & 0.25 & 0.21 & 173.19\tabularnewline
\midrule 
RegNetY-800MF & 224 & 6.3\,M & 800\,M & 76.3\% & 0.31 & 0.22 & 202.66\tabularnewline
\midrule 
ResNet-18 & 224 & 11.7\,M & 1.8\,G & 70.9\% & 0.13 & 0.06 & 158.70\tabularnewline
\midrule 
ResNet-34 & 224 & 21.8\,M & 3.6\,G & 74.4\% & 0.22 & 0.11 & 280.44\tabularnewline
\midrule 
ResNet-50 & 224 & 25.6\,M & 4.1\,G & 77.4\% & 0.40 & 0.20 & 502.43\tabularnewline
\midrule 
ResNet-101 & 224 & 44.5\,M & 7.8\,G & 78.3\% & 0.66 & 0.32 & 937.11\tabularnewline
\midrule 
ResNet-152 & 224 & 60.2\,M & 11.5\,G & 79.2\% & 0.94 & 0.46 & 1261.97\tabularnewline
\midrule 
EfficientNet-B0 & 224 & 5.3\,M & 390\,M & 76.3\% & 0.35 & 0.62 & 160.72\tabularnewline
\midrule 
EfficientNet-B1 & 240 & 7.8\,M & 700\,M & 78.8\% & 0.55 & 1.02 & 254.26\tabularnewline
\midrule 
EfficientNet-B2 & 260 & 9.2\,M & 1.0\,G & 79.8\% & 0.64 & 1.21 & 321.45\tabularnewline
\midrule 
EfficientNet-B3 & 300 & 12.0\,M & 1.8\,G & 81.1\% & 1.12 & 1.86 & 569.30\tabularnewline
\midrule 
EfficientNet-B4 & 380 & 19.0\,M & 4.2\,G & 82.6\% & 2.33 & 3.66 & 1252.79\tabularnewline
\midrule 
EfficientNet-B5 & 456 & 30.0\,M & 9.9\,G & 83.3\% & 4.49 & 6.99 & 2580.25\tabularnewline
\midrule 
EfficientNet-B6 & 528 & 43.0\,M & 19.0\,G & 84.0\% & 7.64 & 12.36 & 4287.81\tabularnewline
\midrule 
EfficientNet-B7 & 600 & 66.0\,M & 37.0\,G & 84.4\% & 13.73 & $\dagger$ & 8615.92\tabularnewline
\midrule 
MobileNetV2-0.25 & 224 & 1.5\,M & 44\,M & 51.8\% & 0.08 & 0.04 & 16.71\tabularnewline
\midrule 
MobileNetV2-0.5 & 224 & 2.0\,M & 108\,M & 64.4\% & 0.10 & 0.05 & 26.99\tabularnewline
\midrule 
MobileNetV2-0.75 & 224 & 2.6\,M & 226\,M & 69.4\% & 0.15 & 0.08 & 49.78\tabularnewline
\midrule 
MobileNetV2-1.0 & 224 & 3.5\,M & 320\,M & 72.0\% & 0.17 & 0.08 & 65.59\tabularnewline
\midrule 
MobileNetV2-1.4 & 224 & 6.1\,M & 610\,M & 74.7\% & 0.24 & 0.12 & 110.70\tabularnewline
\midrule 
MnasNet-1.0 & 224 & 4.4\,M & 330\,M & 74.2\% & 0.17 & 0.11 & 65.50\tabularnewline
\midrule 
DNANet-a & 224 & 4.2\,M & 348\,M & 77.1\% & 0.29 & 0.60 & 157.94\tabularnewline
\midrule 
DNANet-b & 224 & 4.9\,M & 406\,M & 77.5\% & 0.37 & 0.77 & 173.66\tabularnewline
\midrule 
DNANet-c & 224 & 5.3\,M & 466\,M & 77.8\% & 0.37 & 0.81 & 194.27\tabularnewline
\midrule 
DNANet-d & 224 & 6.4\,M & 611\,M & 78.4\% & 0.54 & 1.10 & 248.08\tabularnewline
\midrule 
DFNet-1 & 224 & 8.5\,M & 746\,M & 69.8\% & 0.07 & 0.04 & 82.87\tabularnewline
\midrule 
DFNet-2 & 224 & 18.0\,M & 1.8\,G & 73.9\% & 0.12 & 0.07 & 168.04\tabularnewline
\midrule 
DFNet-2a & 224 & 18.1\,M & 2.0\,G & 76.0\% & 0.19 & 0.09 & 223.20\tabularnewline
\midrule 
OFANet-9ms & 118 & 5.2\,M & 313\,M & 75.3\% & 0.14 & 0.13 & 82.69\tabularnewline
\midrule 
OFANet-11ms & 192 & 6.2\,M & 352\,M & 76.1\% & 0.17 & 0.19 & 94.17\tabularnewline
\midrule 
OFANet-389M(+) & 224 & 8.4\,M & 389\,M & 79.1\% & 0.26 & 0.49 & 116.34\tabularnewline
\midrule 
OFANet-482M(+) & 224 & 9.1\,M & 482\,M & 79.6\% & 0.33 & 0.57 & 142.76\tabularnewline
\midrule 
OFANet-595M(+) & 236 & 9.1\,M & 595\,M & 80.0\% & 0.41 & 0.61 & 150.83\tabularnewline
\midrule 
OFANet-389M{*} & 224 & 8.4\,M & 389\,M & 76.3\% & 0.26 & 0.49 & 116.34\tabularnewline
\midrule 
OFANet-482M{*} & 224 & 9.1\,M & 482\,M & 78.8\% & 0.33 & 0.57 & 142.76\tabularnewline
\midrule 
OFANet-595M{*} & 236 & 9.1\,M & 595\,M & 79.8\% & 0.41 & 0.61 & 150.83\tabularnewline
\midrule 
DenseNet-121 & 224 & 8.0\,M & 2.9\,G & 74.7\% & 0.53 & 0.43 & 395.51\tabularnewline
\midrule 
DenseNet-161 & 224 & 28.7\,M & 7.8\,G & 77.7\% & 1.06 & 0.50 & 991.61\tabularnewline
\midrule 
DenseNet-169 & 224 & 14.1\,M & 3.4\,G & 76.0\% & 0.69 & 0.65 & 490.24\tabularnewline
\midrule 
DenseNet-201 & 224 & 20.0\,M & 4.3\,G & 77.2\% & 0.89 & 1.10 & 642.98\tabularnewline
\midrule 
ResNeSt-50 & 224 & 27.5\,M & 5.4\,G & 81.1\% & 0.76 & $\ddagger$ & 615.77\tabularnewline
\midrule 
ResNeSt-101 & 224 & 48.3\,M & 10.2\,G & 82.3\% & 1.40 & $\ddagger$ & 1130.59\tabularnewline
\midrule 
ZenNet-0.1ms & 224 & 30.1\,M & 1.7\,G & 77.8\% & 0.10 & 0.08 & 181.7 \tabularnewline
\midrule 
ZenNet-0.2ms & 224 & 49.7\,M & 3.4\,G & 80.8\% & 0.20 & 0.16 & 357.4 \tabularnewline
\midrule 
ZenNet-0.3ms & 224 & 85.4\,M & 4.9\,G & 81.5\% & 0.30 & 0.26 & 517.0 \tabularnewline
\midrule 
ZenNet-0.5ms & 224 & 118\,M & 8.3\,G & 82.7\% & 0.50 & 0.41 & 798.7 \tabularnewline
\midrule 
ZenNet-0.8ms & 224 & 183\,M & 13.9\,G & 83.0\% & 0.80 & 0.57 & 1365.0 \tabularnewline
\midrule 
ZenNet-1.2ms & 224 & 180\,M & 22.0\,G & 83.6\% & 1.20 & 0.85 & 2051.1 \tabularnewline
\midrule 
ZenNet-400M-SE & 224 & 5.7\,M & 410\,M & 78.0\% & 0.248 & 0.39 & 87.9 \tabularnewline
\midrule 
ZenNet-600M-SE & 224 & 7.1\,M & 611\,M & 79.1\% & 0.358 & 0.52 & 128.6 \tabularnewline
\midrule 
ZenNet-900M-SE & 224 & 13.3\,M & 926\,M & 80.8\% & 0.55 & 0.55 & 215.68\tabularnewline
\bottomrule \\
  \caption{One big table of all networks referred in this work. \\
  $+$: OFANet trained using supernet parameters as initialization. \\
  $*$: OFANet trained from scratch. We adopt this setting for fair comparison.\\
  $\dagger$: fail to run due to out of memory. \\
  $\ddagger$: official model implementation not supported by TensorRT.}
  \label{tab:one-big-table}
\end{longtable}

\clearpage \newpage

\section{Detail Structure of ZenNets}

We list detail structure in Table \ref{tab:struct-ZenNet-0.1ms}, \ref{tab:struct-ZenNet-0.2ms}, \ref{tab:struct-ZenNet-0.3ms}, \ref{tab:struct-ZenNet-0.5ms}, \ref{tab:struct-ZenNet-0.8ms}, \ref{tab:struct-ZenNet-1.2ms}, \ref{tab:struct-ZenNet-400M-SE}, \ref{tab:struct-ZenNet-600M-SE}, \ref{tab:struct-ZenNet-900M-SE}, \ref{tab:struct-ZenNet-1M}, \ref{tab:struct-ZenNet-2M}.

The 'block' column is the block type. 'Conv' is the standard convolution layer followed by BN and RELU. 'Res' is the residual block used in ResNet-18. 'Btn' is the residual bottleneck block used in ResNet-50. 'MB' is the MobileBlock used in MobileNet and EfficientNet. To be consistent with 'Btn' block, each 'MB' block is stacked by two MobileBlocks. That is, the kxk full convolutional layer in 'Btn' block is replaced by depth-wise convolution in 'MB' block. 'kernel' is the kernel size of kxk convolution layer in each block. 'in', 'out' and 'bottleneck' are numbers of input channels, output channels and bottleneck channels respectively. 'stride' is the stride of current block. '\# layers' is the number of duplication of current block type.

\begin{table}[!h]
 \begin{center}
   \begin{tabular}{lllllll}
     \toprule 
     block & kernel & in & out & stride & bottleneck & \# layers\tabularnewline
     \midrule
     \midrule
     Conv & 3 & 3 & 24 & 2 & - & 1\tabularnewline
     \midrule
     Res & 3 & 24 & 32 & 2 & 64 & 1\tabularnewline
     \midrule
     Res & 5 & 32 & 64 & 2 & 32 & 1\tabularnewline
     \midrule
     Res & 5 & 64 & 168 & 2 & 96 & 1\tabularnewline
     \midrule
     Btn & 5 & 168 & 320 & 1 & 120 & 1\tabularnewline
     \midrule
     Btn & 5 & 320 & 640 & 2 & 304 & 3\tabularnewline
     \midrule
     Btn & 5 & 640 & 512 & 1 & 384 & 1\tabularnewline
     \midrule
     Conv & 1 & 512 & 2384 & 1 & - & 1\tabularnewline
     \bottomrule
     \end{tabular}
 \end{center}
 \caption{ZenNet-0.1ms}
 \label{tab:struct-ZenNet-0.1ms}
\end{table}

\begin{table}[!h]
 \begin{center}
   \begin{tabular}{lllllll}
     \toprule 
     block & kernel & in & out & stride & bottleneck & \# layers\tabularnewline
     \midrule
     \midrule
     Conv & 3 & 3 & 24 & 2 & - & 1\tabularnewline
     \midrule
     Btn & 5 & 24 & 32 & 2 & 32 & 1\tabularnewline
     \midrule
     Btn & 7 & 32 & 104 & 2 & 64 & 1\tabularnewline
     \midrule
     Btn & 5 & 104 & 512 & 2 & 160 & 1\tabularnewline
     \midrule
     Btn & 5 & 512 & 344 & 1 & 192 & 1\tabularnewline
     \midrule
     Btn & 5 & 344 & 688 & 2 & 320 & 4\tabularnewline
     \midrule
     Btn & 5 & 688 & 680 & 1 & 304 & 3\tabularnewline
     \midrule
     Conv & 1 & 680 & 2552 & 1 & - & 1\tabularnewline
     \bottomrule
     \end{tabular}
 \end{center}
 \caption{ZenNet-0.2ms}
 \label{tab:struct-ZenNet-0.2ms}
\end{table}

\begin{table}[!h]
 \begin{center}
   \begin{tabular}{lllllll}
     \toprule 
     block & kernel & in & out & stride & bottleneck & \# layers\tabularnewline
     \midrule
     \midrule
Conv & 3 & 3 & 24 & 2 & - & 1\tabularnewline
\midrule
Btn & 5 & 24 & 64 & 2 & 32 & 1\tabularnewline
\midrule
Btn & 3 & 64 & 128 & 2 & 128 & 1\tabularnewline
\midrule
Btn & 7 & 128 & 432 & 2 & 128 & 1\tabularnewline
\midrule
Btn & 5 & 432 & 272 & 1 & 160 & 1\tabularnewline
\midrule
Btn & 5 & 272 & 848 & 2 & 384 & 4\tabularnewline
\midrule
Btn & 5 & 848 & 848 & 1 & 320 & 3\tabularnewline
\midrule
Btn & 5 & 848 & 456 & 1 & 320 & 3\tabularnewline
\midrule
Conv & 1 & 456 & 6704 & 1 & - & 1\tabularnewline
     \bottomrule
     \end{tabular}
 \end{center}
 \caption{ZenNet-0.3ms}
 \label{tab:struct-ZenNet-0.3ms}
\end{table}

\begin{table}[!h]
 \begin{center}
   \begin{tabular}{lllllll}
     \toprule 
     block & kernel & in & out & stride & bottleneck & \# layers\tabularnewline
     \midrule
     \midrule
Conv & 3 & 3 & 8 & 2 & - & 1\tabularnewline
\midrule
Btn & 7 & 8 & 64 & 2 & 32 & 1\tabularnewline
\midrule
Btn & 3 & 64 & 152 & 2 & 128 & 1\tabularnewline
\midrule
Btn & 5 & 152 & 640 & 2 & 192 & 4\tabularnewline
\midrule
Btn & 5 & 640 & 640 & 1 & 192 & 2\tabularnewline
\midrule
Btn & 5 & 640 & 1536 & 2 & 384 & 4\tabularnewline
\midrule
Btn & 5 & 1536 & 816 & 1 & 384 & 3\tabularnewline
\midrule
Btn & 5 & 816 & 816 & 1 & 384 & 3\tabularnewline
\midrule
Conv & 1 & 816 & 5304 & 1 & - & 1\tabularnewline
     \bottomrule
     \end{tabular}
 \end{center}
 \caption{ZenNet-0.5ms}
 \label{tab:struct-ZenNet-0.5ms}
\end{table}

\begin{table}[!h]
 \begin{center}
   \begin{tabular}{lllllll}
     \toprule 
     block & kernel & in & out & stride & bottleneck & \# layers\tabularnewline
     \midrule
     \midrule
     Conv & 3 & 3 & 16 & 2 & - & 1\tabularnewline
     \midrule
     Btn & 5 & 16 & 64 & 2 & 64 & 1\tabularnewline
     \midrule
     Btn & 3 & 64 & 240 & 2 & 128 & 2\tabularnewline
     \midrule
     Btn & 7 & 240 & 640 & 2 & 160 & 3\tabularnewline
     \midrule
     Btn & 7 & 640 & 768 & 1 & 192 & 4\tabularnewline
     \midrule
     Btn & 5 & 768 & 1536 & 2 & 384 & 5\tabularnewline
     \midrule
     Btn & 5 & 1536 & 1536 & 1 & 384 & 3\tabularnewline
     \midrule
     Btn & 5 & 1536 & 2304 & 1 & 384 & 5\tabularnewline
     \midrule
     Conv & 1 & 2304 & 4912 & 1 & - & 1\tabularnewline
     \bottomrule
     \end{tabular}
 \end{center}
 \caption{ZenNet-0.8ms}
 \label{tab:struct-ZenNet-0.8ms}
\end{table}

\begin{table}[!h]
 \begin{center}
   \begin{tabular}{lllllll}
     \toprule 
     block & kernel & in & out & stride & bottleneck & \# layers\tabularnewline
     \midrule
     \midrule
     Conv & 3 & 3 & 32 & 2 & - & 1\tabularnewline
     \midrule
     Btn & 5 & 32 & 80 & 2 & 32 & 1\tabularnewline
     \midrule
     Btn & 7 & 80 & 432 & 2 & 128 & 5\tabularnewline
     \midrule
     Btn & 7 & 432 & 640 & 2 & 192 & 3\tabularnewline
     \midrule
     Btn & 7 & 640 & 1008 & 1 & 160 & 5\tabularnewline
     \midrule
     Btn & 7 & 1008 & 976 & 1 & 160 & 4\tabularnewline
     \midrule
     Btn & 5 & 976 & 2304 & 2 & 384 & 5\tabularnewline
     \midrule
     Btn & 5 & 2304 & 2496 & 1 & 384 & 5\tabularnewline
     \midrule
     Conv & 1 & 2496 & 3072 & 1 & - & 1\tabularnewline
     \bottomrule
     \end{tabular}
 \end{center}
 \caption{ZenNet-1.2ms}
 \label{tab:struct-ZenNet-1.2ms}
\end{table}

\begin{table}[!h]
 \begin{center}
   \begin{tabular}{llllllll}
     \toprule 
     block & kernel & in & out & stride & bottleneck & expansion & \# layers\tabularnewline
     \midrule
     \midrule
Conv & 3 & 3 & 16 & 2 & - & - & 1 \tabularnewline
\midrule
MB & 7 & 16 & 40 & 2 & 40 & 1 & 1 \tabularnewline
\midrule
MB & 7 & 40 & 64 & 2 & 64 & 1 & 1 \tabularnewline
\midrule
MB & 7 & 64 & 96 & 2 & 96 & 4 & 5 \tabularnewline
\midrule
MB & 7 & 96 & 224 & 2 & 224 & 2 & 5 \tabularnewline
\midrule
Conv & 1 & 224 & 2048 & 1 & - & - & 1 \tabularnewline
     \bottomrule
     \end{tabular}        
 \end{center}
 \caption{ZenNet-400M-SE}
 \label{tab:struct-ZenNet-400M-SE}
\end{table}

\begin{table}[!h]
 \begin{center}
   \begin{tabular}{llllllll}
     \toprule 
     block & kernel & in & out & stride & bottleneck & expansion & \# layers\tabularnewline
     \midrule
     \midrule
Conv & 3 & 3 & 24 & 2 & - & - & 1 \tabularnewline
\midrule
MB & 7 & 24 & 48 & 2 & 48 & 1 & 1 \tabularnewline
\midrule
MB & 7 & 48 & 72 & 2 & 72 & 2 & 1 \tabularnewline
\midrule
MB & 7 & 72 & 96 & 2 & 88 & 6 & 5 \tabularnewline
\midrule
MB & 7 & 96 & 192 & 2 & 168 & 4 & 5 \tabularnewline
\midrule
Conv & 1 & 192 & 2048 & 1 & - & - & 1 \tabularnewline
     \bottomrule
     \end{tabular}        
 \end{center}
 \caption{ZenNet-600M-SE}
 \label{tab:struct-ZenNet-600M-SE}
\end{table}

\begin{table}[!h]
 \begin{center}
   \begin{tabular}{llllllll}
     \toprule 
     block & kernel & in & out & stride & bottleneck & expansion & \# layers\tabularnewline
     \midrule
     \midrule
Conv & 3 & 3 & 16 & 2 & - & - & 1 \tabularnewline
\midrule
MB & 7 & 16 & 48 & 2 & 72 & 1 & 1 \tabularnewline
\midrule
MB & 7 & 48 & 72 & 2 & 64 & 2 & 3 \tabularnewline
\midrule
MB & 7 & 72 & 152 & 2 & 144 & 2 & 3 \tabularnewline
\midrule
MB & 7 & 152 & 360 & 2 & 352 & 2 & 4 \tabularnewline
\midrule
MB & 7 & 360 & 288 & 1 & 264 & 4 & 3 \tabularnewline
\midrule
Conv & 1 & 288 & 2048 & 1 & - & - & 1 \tabularnewline
     \bottomrule
     \end{tabular}        
 \end{center}
 \caption{ZenNet-900M-SE}
 \label{tab:struct-ZenNet-900M-SE}
\end{table}

\begin{table}[!h]
 \begin{center}
   \begin{tabular}{lllllll}
     \toprule 
     block & kernel & in & out & stride & bottleneck & \# layers\tabularnewline
     \midrule
     \midrule 
     Conv & 3 & 3 & 88 & 1 & - & 1\tabularnewline
     \midrule 
     Btn & 7 & 88 & 120 & 1 & 16 & 1\tabularnewline
     \midrule 
     Btn & 7 & 120 & 192 & 2 & 16 & 3\tabularnewline
     \midrule 
     Btn & 5 & 192 & 224 & 1 & 24 & 4\tabularnewline
     \midrule 
     Btn & 5 & 224 & 96 & 2 & 24 & 2\tabularnewline
     \midrule 
     Btn & 3 & 96 & 168 & 2 & 40 & 3\tabularnewline
     \midrule 
     Btn & 3 & 168 & 112 & 1 & 48 & 3\tabularnewline
     \midrule 
     Conv & 1 & 112 & 512 & 1 & - & 1\tabularnewline
     \bottomrule
     \end{tabular}
 \end{center}
 \caption{ZenNet-1.0M for CIFAR-10/CIFAR-100}
 \label{tab:struct-ZenNet-1M}
\end{table}

\begin{table}[!h]
 \begin{center}
   \begin{tabular}{lllllll}
     \toprule 
     block & kernel & in & out & stride & bottleneck & \# layers\tabularnewline
     \midrule
     \midrule 
     Conv & 3 & 3 & 32 & 1 & - & 1\tabularnewline
     \midrule 
     Btn & 5 & 32 & 120 & 1 & 40 & 1\tabularnewline
     \midrule 
     Btn & 5 & 120 & 176 & 2 & 32 & 3\tabularnewline
     \midrule 
     Btn & 7 & 176 & 272 & 1 & 24 & 3\tabularnewline
     \midrule 
     Btn & 3 & 272 & 176 & 1 & 56 & 3\tabularnewline
     \midrule 
     Btn & 3 & 176 & 176 & 1 & 64 & 4\tabularnewline
     \midrule 
     Btn & 5 & 176 & 216 & 2 & 40 & 2\tabularnewline
     \midrule 
     Btn & 3 & 216 & 72 & 2 & 56 & 2\tabularnewline
     \midrule 
     Conv & 1 & 72 & 512 & 1 & - & 1\tabularnewline
     \bottomrule
     \end{tabular}        
 \end{center}
 \caption{ZenNet-2.0M for CIFAR-10/CIFAR-100}
 \label{tab:struct-ZenNet-2M}
\end{table}

\end{document}